\documentclass[10pt]{article}
\usepackage{amsmath, amssymb, amsthm, algorithm, algorithmic,graphicx,epstopdf}
\usepackage{url}
\usepackage{bm}
\usepackage{enumerate}
\usepackage{booktabs}
\usepackage{color}
\usepackage{amsfonts}
\usepackage{array}
\usepackage{mdwmath}
\usepackage{url}
\usepackage{multirow}
\usepackage{caption} \usepackage{subcaption}

\renewcommand{\Re}{\mathbb{R}}
\usepackage{tikz}

\newcommand{\norm}[1]{\left\|#1\right\|}

\newtheorem{theorem}{Theorem}
\newtheorem{lem}[theorem]{Lemma}

\newtheorem{sigmodel}[theorem]{Asssumption}
\newtheorem{corollary}[theorem]{Corollary}
\newtheorem{definition}[theorem]{Definition}
\newtheorem{remark}[theorem]{Remark}
\newtheorem{example}[theorem]{Example}

\newcommand{\bi}{\begin{itemize}}
\newcommand{\ei}{\end{itemize}}
\newcommand{\ben}{\begin{enumerate}}
\newcommand{\een}{\end{enumerate}}

\newcommand{\bean}{\begin{eqnarray*} }
\newcommand{\eean}{\end{eqnarray*} }
\newcommand{\bea}{\begin{eqnarray} }
\newcommand{\eea}{\end{eqnarray} }
\newcommand{\ba}{\begin{align*} }
\newcommand{\ea}{\end{align*} }

\newcommand{\nn}{\nonumber}

\newcommand{\n}{{\cal N}}
\newcommand{\nois}{g}
\newcommand{\epsbnd}{d(\alpha)}
\newcommand{\epsden}{d_{\mathrm{denom}}(\alpha)}

\newcommand{\rest}{\mathrm{rest}}

\newcommand{\xhat}{\bm{\hat{x}}}

\newcommand{\bl}{\begin{frame}}
\newcommand{\el} {\end{frame}}

\renewcommand\thetheorem{\arabic{section}.\arabic{theorem}}

\newcommand{\wt}{\bm{w}_t}
\newcommand{\mt}{\bm{m}_t}
\newcommand{\xt}{\bm{x}_t}
\newcommand{\x}{\bm{x}}

\renewcommand{\l}{\bm{\ell}}

\newcommand{\lt}{\bm{\ell}_t}
\newcommand{\lhat}{\hat{\bm{\ell}}}

\newcommand{\yt}{\bm{y}_t}
\newcommand{\y}{\bm{y}}

\newcommand{\w}{\bm{w}}
\renewcommand{\v}{{\bm{v}}}
\newcommand{\vt}{\bm{v}_t}

\renewcommand{\a}{\bm{a}}
\newcommand{\e}{\bm{e}} 

\newcommand{\et}{\bm{e}_t}

\newcommand{\at}{\bm{a}_t}

\newcommand{\I}{\bm{I}}

\newcommand{\Lam}{\bm{\Lambda}}

\newcommand{\T}{\mathcal{T}}

\newcommand{\D}{\bm{D}}
\newcommand{\A}{\bm{A}}

\renewcommand{\P}{\bm{P}}

\newcommand{\Phat}{\hat{\bm{P}}}

\newcommand{\Span}{\operatorname{span}} 

\newcommand{\E}{\mathbb{E}}

\newcommand{\SE}{\mathrm{SE}}

\newcommand{\M}{\bm{M}}

\renewcommand{\L}{\bm{L}}

\newcommand{\matr}[2]{ \left[\begin{array}{cc}
     #1 \\
     #2
   \end{array}
  \right]
  }

\newcommand{\B}{\bm{B}}
\renewcommand{\Re}{\mathbb{R}}

\newcommand{\iidsim}{\overset{\mathrm{iid}}{\sim}} 
\usepackage{tikz}
\usepackage{pgfplots}
\usepgfplotslibrary{groupplots}
\pgfplotstableread[col sep = comma]{figures_new_pca/bound_SE_n100r5.dat}\SEbound
\pgfplotstableread[col sep = comma]{figures_new_pca/bound_SE_n1000r5.dat}\SEboundnew
\pgfplotsset{
        my legend style compare/.style={
            legend style={
                at={(2.35,.85)},
                anchor=north west,
            },
            legend columns=1,
	    legend style={font=\small},
        },
            x tick label style={font =\small, /pgf/number format/1000 sep=},
            y tick label style={font =\small, /pgf/number format/1000 sep=},
        cycle multi list={
        {red, line width=0.6pt, mark=o,mark size=2pt}, 
        {blue, line width=0.6pt, mark=square,mark size=1.5pt}, 
        {teal, line width=0.6pt, mark=triangle,mark size=2.5pt},
		},
}
\pgfplotsset{
  log y ticks with fixed point/.style={
      yticklabel={
        \pgfkeys{/pgf/fpu=true}
        \pgfmathparse{exp(\tick)}%
        \pgfmathprintnumber[fixed relative, precision=3]{\pgfmathresult}
        \pgfkeys{/pgf/fpu=false}
      }
  }
}

\newcommand{\Z}{\bm{Z}}
\newcommand{\bx}{\bm{x}}
\newcommand{\by}{\bm{y}}
\newcommand{\bz}{b}

\begin{document}

\title{Finite Sample Guarantees for PCA in Non-Isotropic and Data-Dependent Noise}
\author{Namrata Vaswani and Praneeth Narayanamurthy \\
\texttt{\{namrata, pkurpadn\}@iastate.edu} \\
Department of Electrical and Computer Engineering \\
Iowa State University}

\date{}
\maketitle

\begin{abstract}
This work obtains novel finite sample guarantees for Principal Component Analysis (PCA). These hold even when the corrupting noise is non-isotropic, and a part (or all of it) is data-dependent. Because of the latter, in general, the noise and the true data are correlated. The results in this work are a significant improvement over those given in our earlier work where this ``correlated-PCA" problem was first studied. In fact, in certain regimes, our results imply that the sample complexity required to achieve subspace recovery error that is a constant fraction of the noise level is near-optimal. Useful corollaries of our result include guarantees for PCA in sparse data-dependent noise and for PCA with missing data. An important application of the former is in proving correctness of the subspace update step of a popular online algorithm for dynamic robust PCA.%
\end{abstract}


\section{Introduction}
Principal Components Analysis (PCA) is among the most frequently used tools for dimension reduction for a wide variety of data analysis applications. Some examples include exploratory data analysis, data classification, image or video retrieval, face recognition, and recommendation system design. 
Given a matrix of observed data, the goal of PCA is to compute a small number of orthogonal directions that contain most of the variability of the data. 
These principal components are easily computed via singular value decomposition (SVD) on the observed data matrix. 
%

PCA is a classical and very well-studied problem. There has been a large amount of work on analyzing PCA, however most existing results for PCA are asymptotic, e.g., see \cite{hong_balzano}, and references therein. 
While asymptotic analysis is useful because it provides limits on what can be done, it is less practically relevant. A very nice work by Nadler \cite{nadler} provides finite sample guarantees for one-dimensional PCA that hold under the spiked covariance model  \cite{spiked_cov}. Spiked covariance model means that true data (``signal") and noise are independent, or at least uncorrelated, and the noise is isotropic (noise power in all directions is equal). A simple example of isotropic noise is noise that is zero mean with a covariance matrix that is a scalar multiple of identity.
%
There is also much new work on analyzing streaming solutions for PCA in the non-asymptotic setting, e.g. \cite{onlinePCA3,streamingPCA2}, however that is a different problem (places memory constraints on the algorithm) and we will not discuss it here. 
All of the above works either assume the spiked covariance model \cite{spiked_cov,hong_balzano,nadler,onlinePCA3} or only analyze one-dimensional PCA \cite{streamingPCA2,hong_balzano,nadler} or both \cite{hong_balzano,nadler}.

Our work obtains novel finite sample (non-asymptotic) guarantees for $r$-dimensional PCA (with $r \ge 1$) that hold even when the corrupting noise is non-isotropic, and, a part, or all, of it is data-dependent. Because of the latter, in general, the data and noise are no longer independent (or even uncorrelated). A special case of this problem was first studied in \cite{corpca_nips} where we called it ``correlated-PCA". As we will explain, the current work significantly improves upon the results of \cite{corpca_nips}.%






{\em Notation. }
We use $\A'$ to denote transpose of a matrix $\A$.
We use $\|\cdot \|_p$ to denote the $l_p$ norm of a vector or the induced $l_p$ norm of a matrix. Most of this paper only uses $l_2$ norm. At a few places, even if the subscript is missing, it refers to the $l_2$ norm.
For a set of indices $\T$, $\I_{\mathcal{T}}$ refers to an $n \times |\mathcal{T}|$ matrix of columns of the identity matrix indexed by entries in $\mathcal{T}$. For a matrix $\bm{A}$,  $\bm{A}_{\mathcal{T}} := \bm{AI}_{\mathcal{T}}$. A tall matrix, $\P$, with orthonormal columns is referred to as a {\em basis matrix}. We use $\Span(\P)$ to denote the span of the columns of the basis matrix $\P$ and we use $\P_\perp$ to denote a basis matrix whose span is the orthogonal complement of $\Span(\P)$. Thus $\P\P' + \P_\perp\P_\perp{}' = \I$.
For two {\em basis matrices} $\Phat$, $\P$, we define the subspace recovery error (SE) as
\[
\SE(\Phat,\P):= \|(\I - \Phat \Phat')\P\|_2.
\]
This measures the sine of the principal angle between column spans of $\Phat$ and $\P$.

We re-use the letters $C, c$ at various places to denote different numerical constants. Also, $\frac{1}{\alpha} \sum_t f(t)$ is often used instead of $\frac{1}{\alpha} \sum_{t=1}^\alpha f(t)$.

{\em Problem Setting. }
We study PCA in the following setting which assumes that the data-dependent component of the noise at each time $t$ depends linearly on the true data (signal) vector at time $t$.
For $t=1,2,\dots, \alpha$, we are given $n$-length observed data vectors, $\yt$, that satisfy
\bea \label{yt_mod}
\y_t:= \lt + \wt + \vt, \text{ where } \lt = \P \at, \  \wt = \M_t \lt, \  \E[\lt \v_t{}']=0, 
\eea
$\P$ is an $n \times r$ basis matrix with $r \ll n$; $\lt$ is the true data (``signal") vector that lies in an $r$ dimensional subspace of $\Re^n$, $\Span(\P)$; $\at$ is its projection into this subspace; $\wt$ is the data-dependent noise component; and $\vt$ is the uncorrelated noise component, i.e., it satisfies $\E[\lt \v_t{}'] = 0$. The data-dependency matrices $\M_t$ are {\em unknown} and such that the signal-noise correlation $\E[\lt \wt{}'] \neq 0$. 
Thus, we also often refer to $\wt$ as ``correlated" noise.
The goal is to estimate $\Span(\P)$. 
Since the matrices $\M_t$ are {\em time-varying}, observe that,  the $\wt$'s taken together, in general, do not lie in a lower dimensional subspace of $\Re^n$.

Data-dependent noise occurs in a large number of applications due to signal reflections or  signal leakage, e.g., in electro-encephalography (EEG) and magneto-encephalography (MEG). It is called interference in these settings. It also often occurs in molecular biology applications when the noise affects the measurement levels through the very same process as the interesting signal \cite{cor_noise_gillberg}. Two other examples where it occurs include PCA with missing data and the subspace update step of the Recursive Projected Compressive Sensing (ReProCS) solution to dynamic robust PCA \cite{rrpcp_perf,rrpcp_aistats}. In these last two examples, the noise also satisfies our required assumption on signal-noise correlation. We explain them in detail in Sec. \ref{apps_missing} and \ref{apps_dynrpca}. 

Non-isotropic noise is even more common. In signal processing literature, it is often referred to as ``colored" noise. One common example of this is the noise in different pixels of an image sequence \cite{beck2009fast}. The variance of the noise is often region-dependent and this is what results in the non-isotropy.



{\em The SVD solution. } 
This computes $\Phat$ as the top $r$ left singular vectors of the observed data matrix $[\y_1, \y_2, \dots, \y_\alpha]$. Equivalently $\Phat$ is the matrix of top $r$ eigenvectors of the $n \times n$ matrix $\D:=\frac{1}{\alpha} \sum_{t=1}^\alpha \yt \yt{}'$. Hence this solution is often also referred to as {\em EVD (eigenvalue decomposition)}.

{\em Related Work. }
To our best knowledge, existing guarantees for PCA other than \cite{nadler,corpca_nips} are asymptotic. 
Also, see the discussion in \cite[Section 1]{onlinePCA3}. 
We discuss these and two other tangentially related works \cite{rel_perturb,versh_cov_est} in Sec. \ref{rel_work}.%


{\em Contributions. }
This work, which builds on work in \cite{corpca_nips}, is the first to study PCA in a non-isotropic and data-dependent noise setting. Our main result (Theorem \ref{mainthm}) shows that it is possible to recover the signal subspace, $\Span(\P)$, with error at most $\varepsilon$ as long as (a) a simple assumption on signal-noise correlation holds, (b)  the ratio between the maximum signal-noise correlation and the minimum signal subspace eigenvalue  is upper bounded; (c) the ratio between the noise power outside the signal subspace and the minimum signal subspace eigenvalue is upper bounded; and (d) the sample complexity, $\alpha$, is lower bounded. All the required bounds depend on $\varepsilon$.
We obtain such a result in two settings -  bounded signal and noise and sub-Gaussian (e.g., Gaussian) signal and noise. In most applications, boundedness is a more practical assumption than Gaussianity since data acquisition devices usually have bounded power.

As compared to the result of \cite{corpca_nips}, our results holds under a much weaker signal-noise correlation assumption {\em and} needs a sample complexity lower bound that is much better than the one given in \cite{corpca_nips}.  In fact, for the only data-dependent noise case studied in \cite{corpca_nips}, our sample complexity bound is near-optimal.
Secondly, we generalize the observed data model to also include an uncorrelated, but possibly non-isotropic, noise term. This is a more practically valid noise model since the noise/corruption is usually not fully data-dependent. Moreover, this allow us to obtain the existing isotropic noise results as special cases.
Lastly, we also provide a simple provably correct method for automatic subspace dimension estimation that does not use knowledge of any model parameter (see Theorem \ref{r_unknown_2}).



{\em Paper Organization. }
We state and discuss the main results in Sec. \ref{main_res}. Related works are discussed in Sec. \ref{rel_work}. In Sec. \ref{apps}, we show how our result can be applied to the problem of PCA in sparse data-dependent noise (PCA-SDDN) and to its two special cases  - PCA in missing data, and the subspace update step of ReProCS for dynamic robust PCA. Numerical experiments backing our theoretical claims are shown in Sec. \ref{expts}. We conclude in Sec. \ref{conclude}.


\section{Main Results} \label{main_res}
In Sec. \ref{assus}, we state our basic assumptions and define a few quantities. In Sec. \ref{uncor} and \ref{ddn}, we state and discuss corollaries for the two special cases - data and only uncorrelated noise and data and only data-dependent noise. We give the most general version of our result (Theorem \ref{mainthm}) in Sec. \ref{mostgen}. This and its corollaries assume that the subspace dimension, $r$, is known. We show how to provably correct estimate $r$ (Theorems \ref{r_unknown_1} and \ref{r_unknown_2}) in Sec. \ref{main_res_unknown_r}.%

\subsection{Basic assumptions} \label{assus}
In this entire paper, we assume the following.
\begin{sigmodel}\label{basic}
The $\lt$'s satisfy $\lt = \P \at$ with $\at$'s being zero mean and mutually independent random variables (r.v.), with diagonal covariance matrix, $\Lam:=\E[\at \at']$.

The $\vt$'s are zero mean and mutually independent  r.v.'s, with covariance matrix $\bm\Sigma_v:= \E[\vt \vt']$. Also, $\E[\lt \v_t{}'] = 0$ for all $t$, i.e., they are uncorrelated.

(Notice that the model on $\lt$ automatically imposes a model on the data-dependent noise component $\wt:= \M_t \lt$.)
\end{sigmodel}

We define a few quantities to state our results compactly.
\begin{definition}
Let
\ben
\item $\lambda^- :=\lambda_{\min}(\Lam)$, $\lambda^+:=\lambda_{\max}(\Lam)$ and
\[
f:=\frac{\lambda^+}{\lambda^-}.
\]

\item Define the following functions of $\bm\Sigma_v$:
\[
\lambda_{v,\P}^-:= \lambda_{\min}(\P'\bm\Sigma_v \P),  \ \lambda_{v,\rest}^+:= \lambda_{\max}(\bm\Sigma_v - \P \P'\bm\Sigma_v \P \P'), \  \lambda_{v,\P,\P_\perp}:= \|\P_\perp{}'\bm\Sigma_v \P\|_2
\]
and
\[
\lambda_v^+:=\|\bm\Sigma_v\|_2.
\]
It is easy to see that $\lambda_{v,\P,\P_\perp} \le  \lambda_{v,\rest}^+$. Also, $\lambda_{v,\rest}^+ \le \lambda_{v}^+$,  $\lambda_{v,\P}^- \le \lambda_{v}^+$.

\item The following factor will used at various places in our results:
\[
\nois: = \max\left( \frac{\lambda_v^+}{\lambda^-}, \sqrt{ \frac{\lambda_v^+}{\lambda^-} f} \right).
\]

\een
We assume that $\lambda_v^+$ and $\lambda^+$ are at most constant ($O(1)$) with $n$.
\label{def1}
\end{definition}

\begin{remark}\label{timevar_Lam}
For notational simplicity, we have let $\Lam$ and $\bm\Sigma_v$ be constant with $t$. However, all our proofs will go through with minor changes if these are time-varying. We explain the changes needed in Remark \ref{timevar_Lam_changes}.
\end{remark}


\begin{sigmodel}[Bounded signal and noise]  
\label{lt_mod}
Assumption \ref{basic} holds and the $\at$'s are element-wise bounded r.v.'s, i.e., there exists a numerical constant, $\eta$, such that,
\[
\max_{j=1,2, \dots r}  \max_t \frac{( \at)_j^2}{\lambda_j(\Lam)} \le \eta.
\]
For example, if $\at$'s are uniformly distributed, then $\eta = 3$. Throughout this paper, $\eta$ will be treated as a numerical constant.

The $\vt$'s are bounded r.v.'s, i.e., there exists an integer $r_v \le C n$ such that
\[
\max_t \|\vt\|_2^2 \le r_v \lambda_v^+.
\]
Here $r_v$ can be interpreted as the {\em ``effective noise dimension" of $\vt$.}
\end{sigmodel}


\begin{sigmodel}[Sub-Gaussian signal and noise]  
\label{lt_mod_subg}
Assumption \ref{basic} holds and $\at$'s and $\vt$'s are sub-Gaussian r.v.'s with sub-Gaussian norms bounded by $C \sqrt{\lambda^+}$ and $C \sqrt{\lambda_v^+}$ respectively. Recall from Definition \ref{def1} that both these are assumed to be $O(1)$ w.r.t. $n$ and so $\at$'s and $\vt$'s are ``nice" sub-Gaussians~\cite{vershynin}.%
\end{sigmodel}

\begin{remark}
We should point out that
Assumption \ref{lt_mod} is not always a special case of Assumption \ref{lt_mod_subg} even though all bounded r.v.'s are formally sub-Gaussian. It is a special case if we assume that $\vt$ is also element-wise bounded by a constant (w.r.t. $n$).
Without this, when $r_v = n$, it is possible that the sub-Gaussian norm of $\vt$ is as large as $\sqrt{n}$ \cite{vershynin}. One example for which this happens is the coordinate distribution \cite[Example 5.25]{vershynin}: $\vt$ is equally likely to take one of  $2n$ possible values $\{\pm \sqrt{n} \e_i \}_{i=1,2,\dots,n}$ where $\e_i$ is the $i$-th column $\I$.
\label{bnd_subg}
\end{remark}

\subsection{Result for only uncorrelated noise case}\label{uncor}
Before stating the most general result, we state its corollaries for only uncorrelated and only correlated noise.  Also, for simplicity, the results in this and the next subsection assume that the subspace dimension $r$ is known. We explain how to provably correctly estimate $r$ automatically in Sec. \ref{main_res_unknown_r}.

\begin{corollary}[uncorrelated non-isotropic noise]
Given data vectors $\yt := \lt + \vt$ for $t=1,2, \dots, \alpha$ with $\lt, \vt$ uncorrelated. Let $\Phat$ denote the matrix of top $r$ eigenvectors of $\D:= \frac{1}{\alpha} \sum_t \yt \yt'$ and define
\[
\epsbnd:= c  \nois \sqrt{\frac{ \max(r_v, r) \log n}{\alpha}} \text{ and } \epsden: = c f \sqrt{\frac{r + \log n}{\alpha}}.
\]
\ben
\item If Assumption \ref{lt_mod} holds, $\alpha^3 > \max(r_v,r) \log n$, and
 $\frac{\lambda_{v,\rest}^+ - \lambda_{v,\P}^-}{\lambda^-} +  \epsbnd  + \epsden < 1$,
then, w.p. at least $1 - 10n^{-10}$,
\[
\SE(\Phat,\P) \le \frac{\frac{\lambda_{v,\P,\P_\perp}}{\lambda^-}   + \epsbnd }{1  -  \frac{\lambda_{v,\rest}^+ - \lambda_{v,\P}^-}{\lambda^-}    - \epsbnd - \epsden }.
\]

\item If  Assumption \ref{lt_mod_subg} holds (instead of Assumption \ref{lt_mod}), then we have the same result as above but with $\epsbnd = \epsbnd_{sG}: = c  \max\left(f,\frac{\lambda_v^+}{\lambda^-}\right) \sqrt{\frac{n}{\alpha}}$.
Also the result now holds w.p. greater than $ 1- 10 \exp(-cn)$ and we do not need $\alpha^3 \ge r \log n$.
\een
\label{only_uncor}
\end{corollary}

\begin{proof}[Proof Outline]
This is a corollary of Theorem \ref{mainthm} given later and proved in the Appendix. We give the main proof idea here. It relies on the Davis-Kahan $\sin \theta$ theorem \cite{davis_kahan} which states the following.
\begin{lem}[Davis-Kahan $\sin \theta$ theorem]\label{sintheta}
Let $\D_0$ be a Hermitian matrix whose span of top $r$ eigenvectors equals $\Span(\P)$. Let $\D$ be the Hermitian matrix with top $r$ eigenvectors  $\Phat$. Then,
\begin{align}
\SE(\Phat,\P) \le \frac{\|(\D-\D_0)\P\|_2}{\lambda_r(\D_0) - \lambda_{r+1}(\D)}  \le  \frac{\|(\D-\D_0)\P\|_2}{\lambda_r(\D_0) - \lambda_{r+1}(\D_0) - \lambda_{\max}(\D-\D_0)}
\label{sintheta_bnd_2}
\end{align}
as long as the denominator is positive. The second inequality follows from the first using Weyl's inequality.
\end{lem}
We apply the above result with $\D_0= \P(\frac{1}{\alpha} \sum_t \at \at' + \P'\bm\Sigma_v \P)\P'$ and use
$\lambda_r(\D_0) \ge \lambda_{v,\P}^- + \lambda^- - \|\frac{1}{\alpha} \sum_t \at \at' - \Lam\|_2 $, $\lambda_{r+1}(\D_0)=0$, and $\D-\D_0 = \E[\D-\D_0] + (\D-\D_0 - \E[\D- \D_0])$ to simplify the bound. We then use appropriate concentration inequalities to upper bound $\|\D-\D_0 - \E[\D-\D_0]\|_2$ and $\|\frac{1}{\alpha} \sum_t \at \at' - \Lam\|_2$. Matrix Bernstein \cite{tail_bound} is used for the former and Vershynin's sub-Gaussian result \cite[Theorem 5.39]{vershynin} is used for the latter. Since the latter involves $r \times r$ matrices, this gives a better bound - $\epsden$ defined above - than matrix Bernstein would give for this term.
%
%
\end{proof}

A further corollary of the above result essentially recovers the subspace error bound given in \cite{nadler} for the case of isotropic independent noise (spiked covariance model). This follows because, when $\bm\Sigma_v = \lambda_v^+ \I$, $\lambda_{v,\P,\P_\perp}=0$ and $\lambda_{v,\rest}^+ = \lambda_{v,\P}^- = \lambda_v^+$.
The result of \cite{nadler} also assumed $r=1$ and Gaussianity.

\begin{corollary}[uncorrelated isotropic noise]
In the setting of Corollary \ref{only_uncor}, if $\bm\Sigma_v =  \lambda_v^+ \I$, then the following simpler result holds: if $\alpha^3 > \max(r_v,r) \log n$, and $ \epsbnd + \epsden< 0.95$,
then, w.p. at least $1 - 10n^{-10}$, $\SE(\Phat,\P) \le \frac{\epsbnd }{1 - \epsbnd - \epsden}.$ In the sub-Gaussian case, $\epsbnd \equiv \epsbnd_{sG}$.%
%
\label{spiked_cov}
\end{corollary}

From Corollary \ref{spiked_cov}, it is clear that, in case of isotropic noise, to achieve subspace error below $\epsilon$ we  only need a lower bound on sample complexity $\alpha$. However, in the general case (Corollary \ref{only_uncor}), we need this sample complexity bound {\em and} an extra assumption such as the following:
\begin{align}
& \lambda_{v,\rest}^+ - \lambda_{v,\P}^- < 0.5 \lambda^-, \text{ and }  2 \lambda_{v,\P, \P_\perp} < 0.5 \epsilon \lambda^-.
\label{noise_bnd_v}
\end{align}
To understand why more assumptions are needed in the general case, observe that $\E[\D] = \P \Lam \P' + \bm\Sigma_v$. If $\bm\Sigma_v = c \I$ (isotropic noise), then the span of top $r$ eigenvectors of $\E[\D]$ is equal to $\Span(\P)$. Thus, as long as $\alpha$ is large enough (sample complexity bound holds), by the $\sin\theta$ theorem stated above, the same will be approximately true for $\Span(\Phat)$ which is the span of top $r$ eigenvectors of $\D$. However, when the noise is not isotropic, this is no longer the case. Without extra assumptions, the span of top eigenvectors of $\E[\D]$ can be very different from $\Span(\P)$. We give a simple example below.
\begin{example}
Suppose that $\bm\Sigma_v = (1.2\lambda^-) (\P_\perp)_1 (\P_\perp)_1{}'$ where $(\P_\perp)_1$ is any one direction from $\Span(\P_\perp)$; thus $\P'(\P_\perp)_1=0$. With this,
\[
\E[\D] = [\P \ (\P_\perp)_1] \matr{\Lam & \ }{ \ & 1.2\lambda^-} \matr{\P'}{(\P_\perp)_1{}'}
\]
(in the above expression, eigenvalues are not in decreasing order).
Since $1.2\lambda^- > \lambda^-$, it is clear that the top $r$ eigenvectors of $\E[\D]$ will be $[\P_1, \P_2, \dots, \P_{r-1}, (\P_\perp)_1]$ (this statement assumes $\lambda_{r-1}(\Lam) > \lambda^-$). Thus their span will be orthogonal to $\P_r$.
As a result the $\SE$ between this span and $\Span(\P)$ will be one. Hence when $\alpha$ is large enough so that $\|\D- \E[\D]\|_2$ is small with high probability (whp), then $\SE(\Phat,\P)$ will also be close to one. To be precise, the following can be shown.

Suppose that $\|\D-\E[\D]\| \le \epsilon \lambda^-$ for any $\epsilon < 0.01$ and that $\lambda_{r-1}(\Lam) \ge 1.1 \lambda^-$. Then $\SE(\Phat,\P) \ge 1 - 11.1 \epsilon$.

To see why this holds, let $\P_{\E D}:=[\P_1, \P_2, \dots, \P_{r-1}, (\P_\perp)_1]$ denote the top $r$ eigenvectors of $\E[\D]$.
Notice that $\lambda_r(\E[\D]) = \max(1.2 \lambda^-, \lambda_{r-1}(\Lam)) \ge 1.1 \lambda^-$ and $\lambda_{r+1}(\E[\D]) = \lambda^-$.
Using Davis-Kahan,
\[
\SE(\Phat,\P_{\E D}) \le \frac{\|\D - \E[\D]\|_2}{\lambda_r(\E[\D]) - \lambda_{r+1}(\E[\D]) - \|\D - \E[\D]\|_2}
\le \frac{\epsilon\lambda^-}{1.1 \lambda^- - \lambda^- - \epsilon\lambda^-} \le 11.1 \epsilon
\]
Using this, triangle inequality, $\P_{\E D}{}'\P_r = 0$ and $\SE(\P_{\E D}{},\Phat) = \SE(\Phat,\P_{\E D}{})$ (this holds since both $\Phat$ and $\P_{\E D}$ have the same dimension),
\bea
&& \SE(\Phat,\P) \ge \SE(\Phat, \P_r) \ge 1 - \|\Phat'\P_r\|_2 = 1 - \|\Phat'(\I - \P_{\E D}{} \P_{\E D}{}') \P_r\|_2  \nn \\
&& \ge 1 - \SE(\P_{\E D}{},\Phat) = 1 - \SE(\Phat,\P_{\E D}{}) \ge 1 - 11.1 \epsilon. \nn
\eea
\label{bad_Sigmav}
\end{example}

For the above example, $\lambda_{v,\P}^- = 0$ while $\lambda_{v,\rest}^+ = 1.2 \lambda^-$ and hence \eqref{noise_bnd_v} does not hold. Because of this, the expected value of the average energy of $\yt$'s in a direction outside $\Span(\P)$ is larger than that in a direction that is in $\Span(\P)$ and this is what causes $\SE(\Phat,\P)$ to be large.
Assuming  \eqref{noise_bnd_v} helps ensure that the above does not happen. It also ensures that the maximal correlation between a component of the projection of $\yt$'s in $\Span(\P)$ and that in $\Span(\P_\perp)$ is small (bound on $\lambda_{v,\P,\P_\perp}$).

{\em Sample complexity. } Consider the required lower bound on $\alpha$ to achieve error $\epsilon$. In the bounded case, our result needs \\ $\alpha \ge C \max\left( \frac{\nois^2}{\epsilon^2}  \max(r_v,r) \log n, f^2 (r + \log n) \right)$. In the sub-Gaussian case, it needs $\alpha \ge C \frac{\nois^2}{\epsilon^2} n$.
Since the subspace dimension is $r$, the minimum number of samples required in any setting is $r$ (this number suffices only when there is no noise outside $\Span(\P)$ and all observations are linearly independent).
Thus, if $r_v$ is small, e.g., if $r_v \in O(r)$, the sample complexity $\alpha$ required to achieve error $\epsilon$ that is a constant fraction of $\nois$ is only $(\log n)$ times more, i.e., it is nearly optimal. 

On the other hand, $r_v$ can be as larger as $n$. If $r_v = C n$, and if $\vt$ is also {\em element-wise} bounded, then it is a ``nice" sub-Gaussian and we can use the sub-Gaussian case result to conclude that, in this case, the required sample complexity is $C n$ (and not $C n \log n$ as predicted by the bounded case result).
However, if $r_v = Cn$ and no other assumption is placed on $\vt$, then, it is not guaranteed to be a ``nice" sub-Gaussian (see Remark \ref{bnd_subg}). In this case, the required sample complexity will indeed be $C (n \log n)$. We discuss this point further in Sec. \ref{rel_work}, where we connect it to the results of \cite{versh_cov_est}.

\begin{remark}\label{tighten_subg}
It may be possible to reduce the required sample complexity lower bound for the sub-Gaussian case in settings where $r_v \ll n$, e.g., if $r_v = C r$, or in the data-dependent noise case discussed below in Sec. \ref{ddn}. For unbounded noise, we can define $r_v$ as $\E[\|\vt\|_2^2]/\lambda_v^+$. In our current proof for the sub-Gaussian case, we use Vershynin's sub-Gaussian result for obtaining all the concentration bounds. However we can try to replace this by an approach motivated by the proof of \cite[Theorem 4.1]{pr_altmin} that relies on the intuition that sub-Gaussian r.v.'s are bound whp. Thus, one can first work with a truncated sub-Gaussian, in our case, truncate each entry to $\sqrt{\log n}$ and use matrix Bernstein, and then deal with the extra errors introduced by this truncation. 
With this approach, we will get an extra factor of $n^{-c}$ in the $\SE$ bound. For large enough $n$, this extra factor is negligible. The advantage will be that we may only need $\alpha \ge C r_v \log^3 n$ instead of $\alpha \ge C n$.
\end{remark}

\subsection{Result for only data-dependent noise case} \label{ddn} \

\begin{corollary}[Only data-dependent noise]
Given data vectors $\yt := \lt + \wt$ with $\wt = \M_t \lt$, $t=1,2, \dots, \alpha$. Let $\Phat$ denote the matrix of top $r$ eigenvectors of $\D:= \frac{1}{\alpha} \sum_t \yt \yt'$ and, for a scalar $q$, let
\[
\epsbnd:= c  q f \sqrt{\frac{r \log n}{\alpha}} \text{ and } \epsden: = c f \sqrt{\frac{r + \log n}{\alpha}}.
\]
\ben \item
If Assumption \ref{lt_mod} holds, $\alpha^3 > (r \log n)$, and if, for scalars $q, \bz < 1$ satisfying $3\sqrt{\bz} q f  + \epsbnd + \epsden < 1$, the  matrices $\M_t$ can be decomposed as $\M_t = \M_{2,t} \M_{1,t}$  with
$\M_{1,t}$ being such that
\bea
\max_t \|\M_{1,t} \P\|_2 \le q
\label{M1t_bnd}
\eea
and
$\M_{2,t}$ being such that $\|{\bm{M}_{2,t}}\|_2 \le 1$ but 
\bea
 \left\|\frac{1}{\alpha} \sum_{t =1}^\alpha {\bm{M}_{2,t}} {\bm{M}_{2,t}}' \right\|_2  \le \bz < 1,
\label{M2t_bnd}
\eea
\label{Mt_cond}
then, w.p. at least $1 - 10n^{-10}$,
\[
\SE(\Phat,\P) \le \frac{ \sqrt{\bz} (2q + q^2) f  + \epsbnd }{1
 - \sqrt{\bz} (2q + q^2) f  - \epsbnd - \epsden}.
\]
\item If  Assumption \ref{lt_mod_subg} holds (instead of Assumption \ref{lt_mod}), then we have the same result as above but with $\epsbnd = \epsbnd_{sG}: = c f \sqrt{\frac{n}{\alpha}}.$
\een
\label{only_cor}
\end{corollary}


\begin{proof}[Proof Outline]
This is a corollary of Theorem \ref{mainthm}. The proof idea is similar to that of Corollary \ref{only_uncor}. In this case we apply the $\sin \theta$ theorem with $\D_0 = \P(\frac{1}{\alpha}\sum_t \at \at{}') \P'$.
Also, we need to carefully bound $\|\frac{1}{\alpha} \sum_{t=1}^\alpha \E[\lt \wt{}']\|_2$ and $\|\frac{1}{\alpha} \sum_{t=1}^\alpha \E[\wt \wt{}']\|_2$; this is done in \eqref{bnd_avg_sig_noise_cor} below.%
%
\end{proof}

{\em Data-dependent noise. }
Observe a few things about data-dependent noise. First, it is clearly non-isotropic and hence a condition that ensures that the noise is small compared to $\lambda^-$ is needed.
%
Second, because of the assumed linear dependency on $\lt$, and because \eqref{M1t_bnd} holds, the noise power depends linearly on maximum signal power, $\lambda^+$: we have $\|\E[\wt \wt']\|_2 \le q^2 \lambda^+$. Here $q^2$ can be interpreted as a bound on the {\em noise-to-signal ratio}.
Third, the signal-noise correlation is nonzero. In fact, its bound also linearly depends on $\lambda^+$: we have $\|\E[\lt \wt']\|_2 \le q \lambda^+$. Since $q<1$, the latter is, in fact, larger than the noise power bound. 
From the proof outline, for large enough $\alpha$, the subspace error bound essentially depends on the ratio $\|\E[\D - \D_0]\|_2 / \lambda^-$. Since the signal-noise correlation is nonzero, this ratio is now bounded by $2 \|\frac{1}{\alpha} \sum_{t=1}^\alpha \E[\l_t \w_t{}']\|_2 + \|\frac{1}{\alpha} \sum_{t=1}^\alpha \E[\w_t \w_t{}']\|_2$ instead of just the second term in the only uncorrelated noise case.
Thus, without an  assumption such as \eqref{M2t_bnd} on the signal-noise correlation, one would require $(2q + q^2) \lambda^+$ to be smaller than $0.45 \epsilon  \lambda^-$ to achieve subspace error below $\epsilon$. 
This is a hard requirement since it implies that $\epsilon$ can never be made smaller than the noise level, $q$. 

Assuming \eqref{M2t_bnd} resolves the above issue. Observe that the subspace error depends on the time-averaged signal-noise correlation and time-averaged noise power, and not on their instantaneous values. The assumption \eqref{M2t_bnd} ensures that the bounds on the time-averaged values of both these are at least $\sqrt{\bz}$ times smaller than the bounds on their instantaneous values: by a careful application of Cauchy-Schwartz inequality, it is not hard to see that (see \eqref{bnd_avg_sig_noise_cor_bnd_1} and \eqref{bnd_avg_sig_noise_cor_bnd_2} in the Appendix):
\begin{align}
& \left\|\frac{1}{\alpha} \sum_{t=1}^\alpha \E[\lt \wt{}'] \right\|_2 \le \sqrt{\bz} q \lambda^+, \text{ and} \\
&  \left\|\frac{1}{\alpha} \sum_{t=1}^\alpha \E[\w_t \w_t{}'] \right\|_2 \le \sqrt{\bz} q^2 \lambda^+.
\label{bnd_avg_sig_noise_cor}
\end{align}
Thus, because \eqref{M2t_bnd} holds, to achieve error $\epsilon$, other than a sample complexity lower bound,
our result just needs
$\sqrt{\bz} (2q + q^2) \lambda^+ < 0.45 \epsilon \lambda^-$. With this, it is possible to achieve subspace recovery error $\epsilon$ that is smaller than $q$ by assuming that $\bz$ is small enough. Of course a lower bound on $\alpha$ that ensures $\epsbnd < 0.1 \epsilon$ and $c  f \sqrt{\frac{r + \log n}{\alpha}} < 0.01$ will also be needed (discussed below).

{\em Examples where \eqref{M1t_bnd} and \eqref{M2t_bnd} holds. }
One class of example situations where \eqref{M2t_bnd} would hold is when the data-dependent noise $\wt$ is sparse (PCA in sparse data-dependent noise). Let $\T_t$ denote its support set. Then, in this case, $\wt = \I_{\T_t} \M_{s,t} \lt$ where $\M_{s,t}$ is a $|\T_t| \times n$ data-dependency matrix.
If we pick $\M_{2,t} = \I_{\T_t}$, then $\sum_t \M_{2,t}\M_{2,t}{}'$ will be a diagonal matrix with $(i,i)$-th entry being equal to the number of time instants $t$ for which the index $i$ is part of the support $\T_t$. Hence $\bz$ will equal the maximum fraction of non-zeros in any row of the matrix $[\w_1, \w_2, \dots, \w_\alpha]$. Thus, in this case, \eqref{M2t_bnd} holds as long as this fraction is smaller than one. It holds with a small enough $\bz$ if this fraction is small enough. 
Moreover, \eqref{M1t_bnd}  will hold as long as $\|\M_{1,t} \P\|_2 = \|\M_{s,t} \P\|_2$ is small. Two examples where this happens are given next.

A special case of the above problem is PCA in missing data. Let $\T_t$ denote the set of missing entries at time $t$. By setting the missing entries to zero, we can write out the observed data vector as $\yt = \lt - \I_{\T_t} \I_{\T_t}{}' \lt$. Thus, in this case,  $\M_{s,t} = - \I_{\T_t}{}'$ and so, $q$ is a bound on $\| \I_{\T_t}{}' \P\|_2$. Thus, for PCA-missing, $q$ will be small if columns of $\P$ are dense vectors and the number of missing entries at each time, $|\T_t|$, is small. We discuss this case further in Sec. \ref{apps_missing}.
Another special case occurs in the subspace update step of ReProCS for dynamic robust PCA \cite{rrpcp_aistats}. We explain this in detail in Sec. \ref{apps_dynrpca}. Briefly, in this case, $\M_{s,t} = \B (\I - \Phat \Phat')$ where $\B$ is a matrix satisfying $\|\B\| \le 1.2$ and $\Phat$ is a previous estimate of $\Span(\P)$ satisfying $\SE(\Phat, \P) < q/1.2  \ll 1$.

Another related class of problems where the above assumptions would hold is if $\wt$ is sparse in a basis or dictionary $\bm{Q}$. Then, $\wt = \bm{Q} \I_{\T_t} \M_{s,t} \lt$. In this we can use $\M_{2,t} = \bm{Q} \I_{\T_t} / \|\bm{Q}\|_2$ and $\M_{1,t} =\|\bm{Q}\|_2  \M_{s,t} $. 

{\em Sample complexity. }
To achieve error below $\epsilon$, in the bounded case, the sample complexity needed is $\alpha \ge C  \max( \frac{q^2f^2}{\epsilon^2} (r \log n), f^2 (r + \log n) )$. 
Thus, if $\epsilon$ is a constant fraction of $q$, the sample complexity needed is just $\alpha \ge C  f^2 (r \log n)$. This is nearly optimal. For constant $f$, this is only $O(\log n)$ times the minimum number of samples required to even define an $r$-dimensional subspace. In the sub-Gaussian case, according to our current result, $O(n)$ samples are needed, however as explained earlier, this can possibly be improved in certain settings.

{\em Comparing Corollaries \ref{only_uncor} and \ref{only_cor}.  }
In both results, the bound on $\SE$ depends on the condition number $f=\lambda^+/\lambda^-$, however, the dependence is much weaker in the uncorrelated noise case. In this case, $f$ only appears in terms that contain the sample complexity $\alpha$. Thus, any $f$ can be dealt with by picking a proportionally larger $\alpha$.
However in the data-dependent noise case, $f$ also appears in a term other than $\epsbnd$ or $\epsden$. To get $\SE$ below $\epsilon$, in this case, one needs $\sqrt{\bz}(2q + q^2) f < 0.45\epsilon$. This is a much stronger requirement since it cannot be ensured just by using more samples $\alpha$. This is needed because the data-dependent noise power depends on $\lambda^+$.


On the other hand, consider the sample complexity  $\alpha$ required to achieve error $\epsilon$ that is a constant fraction of the noise level in the bounded data and noise setting.
In the uncorrelated noise case, this is $C \max(r_v,r)\log n$ whereas in the data-dependent noise case, this is only $C (r \log n)$. 
If $r_v$ is larger than $r$, the former will need more samples.

\subsection{The general result} \label{mostgen}
We now give the most general result.
\begin{theorem}
Given data vectors $\yt := \lt + \wt + \vt$ with $\wt = \M_t \lt$ and $\vt$ and $\lt$ uncorrelated. Let $\Phat$ denote the matrix of top $r$ eigenvectors of $\D:= \frac{1}{\alpha} \sum_t \yt \yt'$ and define
\[
\epsbnd:=
c \sqrt{\eta}
 \max \left( q f \sqrt{\frac{r \log n}{\alpha}} ,  \nois \sqrt{\frac{ \max(r_v, r) \log n}{\alpha}} \right), 
\]
and
\[
\epsden: = c \eta f \sqrt{\frac{r + \log n}{\alpha}}.
\]

\ben
\item If Assumption \ref{lt_mod} holds, $\alpha^3 > \max(r_v,r) \log n$,
\ben
\item if, for a $\bz<1$ and a $q < 1$, the data-dependency matrices $\M_t$ satisfy the assumption given in Corollary \ref{only_cor},
\item and if $\frac{\lambda_{v,\rest}^+ - \lambda_{v,\P}^-}{\lambda^-} + 3\sqrt{\bz} q f + \epsbnd  + \epsden < 1$,
\een
then, w.p. at least $1 - 10n^{-10}$,
\[
\SE(\Phat,\P) \le \frac{\frac{\lambda_{v,\P,\P_\perp}}{\lambda^-} + \sqrt{\bz} (2q + q^2) f + \epsbnd }{1  -  \frac{\lambda_{v,\rest}^+ - \lambda_{v,\P}^-}{\lambda^-}  - \sqrt{\bz} (2q + q^2) f  - \epsbnd - \epsden}.
\]
\item If  Assumption \ref{lt_mod_subg} holds (instead of Assumption \ref{lt_mod}), then we have the same result as above but with $\epsbnd = \epsbnd_{sG}: = c \max\left(\frac{\lambda_v^+}{\lambda^-},f \right)  \sqrt{\frac{n}{\alpha}}$
and we do not need $\alpha^3 \ge r \log n$. Also, the result now holds w.p. greater than $ 1- 10 \exp(-cn)$.
\een
\label{mainthm}
\end{theorem}

\noindent {\em Proof.} See Appendices \ref{proof_mainthm} and \ref{proof_mainthm_subg} for the bounded and sub-Gaussian cases respectively.%


Corollary \ref{only_uncor} follows from the above result by setting $q=0$. Corollary \ref{only_cor} follows by setting $r_v=0$ and $\Sigma_v^+= 0$. In both corollaries, we treat $\eta$ as a numerical constant. 

From Theorem \ref{mainthm}, to get error below $\epsilon$ in the most general case, we need a sample complexity lower bound {\em and} we the signal-noise correlation assumption stated in Corollary \ref{only_cor} to hold with parameters $\bz,q$ that satisfy
\[
\sqrt{\bz} (2q + q^2) f + \frac{\lambda_{v,\rest}^+ - \lambda_{v,\P}^-}{\lambda^-} < 0.5, \text{ and }
\frac{\lambda_{v,\P,\P_\perp}}{\lambda^-} + \sqrt{\bz} (2q + q^2) f < 0.25 \epsilon.
\]
In the bounded case, the sample complexity required is
\\ $\alpha \ge C \max\left( \frac{\nois^2}{\epsilon^2} \max(r_v, r) \log n,
\frac{(qf)^2}{\epsilon^2} r \log n,
f^2 (r + \log n)
\right)$. In the sub-Gaussian case, we need $\alpha \ge C  \frac{\nois^2}{\epsilon^2} n$ although, as discussed earlier in Remark \ref{tighten_subg}, this can possibly be improved.

\begin{remark}\label{timevar_Lam_changes}
If $\Lam$ and $\bm\Sigma_v$ were time-varying, the above result will hold with the following simple changes.

(1) Define ``average" versions of $\lambda^-$ and $\lambda_{v,\P}^-$ as $\bar\lambda^- := \lambda_{\min}(\frac{1}{\alpha} \sum_t \Lam_t)$. Define $\bar\lambda_{v,\P}^-:=\lambda_{\min}( \P'(\frac{1}{\alpha} \sum_t  \bm\Sigma_{v,t}) \P)$.
In the result above, replace $\lambda^-$ and $\lambda_{v,\P}^-$ by their ``average" versions $\bar\lambda^-$ and $\bar\lambda_{v,\P}^-$ respectively.

(2) Define $\lambda_{max}^+ :=\max_t \lambda_{\max}(\Lam_t)$. Similarly define ``max" versions of $\lambda_v^+$, $\lambda_{v,\rest}^+$, and $\lambda_{v,\P,\P_\perp}$.
In the result above, replace $\lambda^+$, $\lambda_v^+$, $\lambda_{v,\rest}^+$, $\lambda_{v,\P,\P_\perp}$ by their ``max" versions $\lambda_{max}^+$, $\lambda_{v,max}^+$, $\lambda_{v,\rest,max}^+$, $\lambda_{v,\P,\P_\perp,max}$ respectively.

(3) Because of the above two changes, $f$ gets replaced by $\lambda_{max}^+/\bar\lambda^-$.
%
\end{remark}


\subsection{Automatically estimating $r$} \label{main_res_unknown_r}
In the above result and its corollaries, we assumed that $r$, which is the signal subspace dimension, is known. In practice however this is usually unknown.
There are two easy and commonly used ways to automatically estimate $r$. 
 The first approach is as done in \cite{corpca_nips}. This computes $\hat{r}$ as the smallest index $j$ for which the $j$-th eigenvalue of  $\D:= \sum_{t=1}^\alpha \yt \yt{}'$ is above a threshold. Thus,
 \bea
\hat{r}:= \arg \min \{j: \lambda_j(\D)  \ge 0.5  \lambda^- \}.
\label{estim_r_1}
 \eea
Notice that this requires knowledge of $\lambda^-$. However, as we will see, this does not require extra assumptions beyond what Theorem \ref{mainthm} already assumes.
An alternate way to estimate  $r$ is by looking for the largest eigen-gap, i.e.,
\bea
\hat{r}:=\arg\max_j [ \lambda_j(\D) - \lambda_{j+1}(\D)].
\label{estim_r_2}
\eea
This does not require knowledge of any model parameter. However, as we see below, this works only under the assumption that consecutive eigenvalues of the matrix $\Lam + \P'\bm\Sigma_v \P$ do not have a large gap. It is also more expensive since it requires computing all eigenvalues of $\D$.

Consider \eqref{estim_r_1}.
To prove that this works, we need to show that $\lambda_r(\D) \ge 0.5  \lambda^-$ and $\lambda_{r+1}(\D) < 0.5  \lambda^-$. Let $\D_0 =\P( \Lam + \P'\bm\Sigma_v \P) \P'$. By Weyl, $\lambda_r(\D) \ge \lambda_r(\D_0) - \|\D - \D_0\|_2 \ge \lambda^- + \lambda_{v,\P}^-  - \|\D - \D_0\|_2$ and $\lambda_{r+1}(\D) \le 0 + \|\D - \D_0\|_2$.
Using \eqref{ED_D0_dif},  \eqref{bnd_avg_sig_noise_cor_bnd_1}, \eqref{bnd_avg_sig_noise_cor_bnd_2}, and Lemma \ref{hp_bnds} from the Appendix,
\[
\|\D-\D_0\|_2  \le \Delta \lambda^-, \text{ where }
\Delta:=\epsden + \epsbnd + 3 \sqrt{\bz} q f + \frac{\lambda_{v,\rest}^+ }{\lambda^-}
\]
where $\epsden, \epsbnd$ are defined in Theorem \ref{mainthm}.
Thus, we have the following result. 
\begin{theorem}[Estimating $r$ using \eqref{estim_r_1}] \label{r_unknown_1}
Assume that Assumption \ref{lt_mod} holds and the assumption on $\M_t$ given in Corollary \ref{only_cor} holds. Let  $\Delta$ be as defined above.
If $\Delta  < \frac{1}{2}$, then w.p. at least $1-10n^{-10}$,  \eqref{estim_r_1} returns the correct estimate of $r$.
This result also holds if Assumption \ref{lt_mod} is replaced by Assumption \ref{lt_mod_subg} as long as we replace $\epsbnd$ by $\epsbnd_{sG}$.
\end{theorem}
Proceeding as above for \eqref{estim_r_2},
\begin{align*}
& \lambda_r(\D) - \lambda_{r+1}(\D) \ge \lambda^-  + \lambda_{v,\P}^- - 2\Delta \lambda^-,  \\
& \text{for $j < r$, }  \lambda_{j}(\D) - \lambda_{j+1}(\D)  \le  (\lambda_j(\Lam + \P'\bm\Sigma_v \P) - \lambda_{j+1}(\Lam + \P'\bm\Sigma_v \P)) + 2 \Delta \lambda^-, \\
& \text{for $j > r$, }  \lambda_{j}(\D) - \lambda_{j+1}(\D)  \le   2\Delta \lambda^-
\end{align*}
Thus we the following result for \eqref{estim_r_2}.
\begin{theorem}[Estimating $r$ using \eqref{estim_r_2}] \label{r_unknown_2}
Assume that Assumption \ref{lt_mod} holds and the assumption on $\M_t$ given in Corollary \ref{only_cor} holds. Let $\Delta$ be as defined above. If
\[
\max_{j<r} (\lambda_j(\Lam + \P'\bm\Sigma_v \P) - \lambda_{j+1}(\Lam + \P'\bm\Sigma_v \P)) \le  (1- 4\Delta) \lambda^- + \lambda_{v,\P}^-
\]
then w.p. at least $1-10n^{-10}$, \eqref{estim_r_2} returns the correct estimate of $r$.
%
This result also holds if Assumption \ref{lt_mod} is replaced by Assumption \ref{lt_mod_subg} as long as we replace $\epsbnd$ by $\epsbnd_{sG}$.
\end{theorem}

\section{Discussion of Related Work} \label{rel_work}
A detailed discussion is given here.


{\em Discussion of \cite{nadler}. } This work was the first to obtain finite sample guarantees for PCA. Its main result, \cite[Theorem 2.1]{nadler}, assumes a spiked covariance model with $r=1$ spike and Gaussianity of both data and noise. It was proved using a different set of concentration bounds and hence its exact form is a little different from our result in this setting. However, if one looks at the dominant terms in its required assumption or in its upper bound on $\sin \theta_{PCA}$, the conclusions are the same as those of our Corollary \ref{spiked_cov}.
In our notation, $\sin \theta_{PCA} \equiv \SE(\Phat,\P)$. First to explain notation, its $p \equiv n$, $n \equiv \alpha$, $\kappa^2 \approx \lambda^-$ (actually $\E[\kappa^2] \equiv \lambda^-$), $\sigma^2 \equiv \lambda_v^+$.
%
%
In our notation, \cite[Theorem 2.1]{nadler} says the following.  
When $n \ge \alpha$, if
$\lambda^- \gtrsim \lambda_v^+ \frac{n}{\alpha}$, then, whp,
$
\SE(\Phat,\P) \lesssim c \sqrt{\frac{\lambda_v^+}{\lambda^-}} \sqrt{\frac{n}{\alpha}}.
$
Here $\gtrsim, \lesssim$ indicate that we are only using the dominant terms from their long expression.

Consider our Corollary \ref{spiked_cov} with $r=1$,  Gaussian data and noise, and $n \ge \alpha$.
Since $r=1$, so $\lambda^+ = \lambda^-$, $f=1$, and $\nois = \max(\frac{\lambda_v^+}{\lambda^-},\sqrt{\frac{\lambda_v^+}{\lambda^-}})$.  Ignoring constants, Corollary \ref{spiked_cov} assumes $\nois <\sqrt{\frac{\alpha}{n}}$. Since $\sqrt{\frac{\alpha}{n}}<1$, the $\max$ in the $\nois$ expression is achieved by the square root term. Thus, in this setting,
Corollary \ref{spiked_cov} says the following: if $2 \sqrt{\frac{\lambda_v^+}{\lambda^-}} \sqrt{\frac{n}{\alpha}}< 1$ (equivalently $\lambda^-  \ge 4 \lambda_v^+ \frac{n}{\alpha}$), then, w.p. at least $1-10\exp(-cn)$, $\SE(\Phat,\P) \le  2 \sqrt{\frac{\lambda_v^+}{\lambda^-}} \sqrt{\frac{n}{\alpha}}$. This is the same as the simplified version of \cite[Theorem 2.1]{nadler} given above.%

Because \cite{nadler} only considered the $r=1$ and spiked covariance model setting, it was able to provide more insight into its guarantees beyond just an $\SE$ upper bound. It showed that its upper bound on subspace error is sharp by also providing an expression for the expected subspace error. Moreover, it provided an approximate expression for the top eigenvector of the sample covariance matrix that is valid when the noise variance is small. For $r>1$, these things are difficult to do. For the setting in our paper (non-isotropic and possibly data-dependent noise), these are even harder to do. We only give an example, Example \ref{bad_Sigmav}, to show that, the subspace error will not be small if a bound on noise power outside $\Span(\P)$ is not assumed.


{\em Comparing Theorem \ref{mainthm} with the result of \cite{corpca_nips}. }
Our result is a significant improvement over that of \cite{corpca_nips} where the correlated-PCA problem was first studied. We include a second uncorrelated noise component in our result which makes the data model more practically valid. Second, we also get results under a general sub-Gaussian data and noise assumption.

To compare with the result of \cite{corpca_nips}, consider Corollary \ref{only_cor} under the bounded assumption. The signal-noise correlation model assumed in it is a significant simplification of the one needed by the result of \cite{corpca_nips}. That result needed $\|\frac{1}{\alpha} \sum_{t =1}^\alpha {\bm{M}_{2,t}} \bm{A}_t {\bm{M}_{2,t}}'\|_2 \le \bz$ to hold for all sets of positive semi-definite (p.s.d.) matrices $\bm{A}_t, \ t=1,2,\dots,\alpha$. This is a much stronger requirement. Our current result only needs this to hold only for $\bm{A}_t=\I$, i.e., it needs \eqref{M2t_bnd} to hold. Consider the sparse $\wt$ example. For this, as explained earlier, \eqref{M2t_bnd} would hold if the fraction of nonzeros in any row of the noise matrix $[\w_1, \w_2, \dots, \w_\alpha]$ is bounded by $\bz$.
On the other hand, it is not clear if the assumption needed by \cite{corpca_nips} holds for this example. The examples given in \cite{corpca_nips} involved much more stringent assumptions on $\T_t$ - the sets $\T_t$ needed to be either mutually disjoint, or mutually disjoint every few frames, or they needed to change in a way to model stop and go object motion in one direction. 

Second, our sample complexity bound is a significant improvement over that of \cite{corpca_nips}. 
If the desired error $\epsilon$ is much larger than $qf$, our required sample complexity is $O(r+\log n)$. If $r \ge c \log n$, this is optimal. In the more common small $\epsilon$ setting, to get the subspace error to below say $ \epsilon= q/4$, we need $\alpha \ge 16 C  f^2  (r \log n) $ samples.
This is also much better than the earlier bound from \cite{corpca_nips} of $\alpha \ge C \frac{f^2}{\epsilon^2} (r^2  \log n) $ which implied that $\alpha \ge 16C \frac{f^2}{q^2} (r^2 \log n) $ was needed to achieve the above subspace error level. 
This older bound had an extra multiplicative factor of $r/q^2$. In our work, we remove the extra $r$ factor by using matrix Bernstein to replace matrix Hoeffding to get high probability bounds on the deviation between time-averaged signal-noise correlation and noise power and their respective expected values.  We remove the extra $1/q^2$ factor by bounding the $r$-th eigenvalue of $\sum_t \lt \lt{}' = \P (\sum_t \at \at{}' )\P'$ by using the sub-Gaussian result of Vershynin (Theorem 5.39 of \cite{vershynin}) to bound the minimum eigenvalue of $\sum_t \at \at{}'$. 
In \cite{corpca_nips}, the authors had used matrix Hoeffding for this term as well.%

{\em Discussion of \cite{rel_perturb}. }  In \cite{rel_perturb} and references therein, the authors study the effect of multiplicative perturbations of Hermitian matrices on their principal subspaces. This line of work provides a tighter bound than Davis-Kahan for the subspace error between principal subspaces of a Hermitian matrix $A$ and of its perturbed version $B A B'$ for a non-singular matrix $B$. However, such results are not applicable for our problem even in the only data-dependent noise case, since $\wt$ satisfies $\wt=\M_t \lt$ where $\M_t$ is time-varying.

{\em Discussion of \cite{versh_cov_est}. } This work develops concentration inequalities for sample covariance matrices in a setting where observed data lie in a Euclidean ball of radius $O(\sqrt{n})$. It obtains new results for the setting where the observed data is bounded but is not a ``nice" sub-Gaussian. They explain that, if observed data satisfies their equation (1.5), then $O(n \log n)$ samples are needed to ensure that the sample covariance matrix is close to its expected value. This matches the sample complexity predicted by our result for the bounded case with $r_v = n$ and without any other assumption. In this case our observed data $\yt = \lt + \wt + \vt$ satisfies equation (1.5) of that paper with $K^2 = \eta \lambda^+(1+q^2) + \lambda_v^+$ and $L^2 = c K^2$.
When  $r_v  = C n$, if we add a mild extra assumption that the $k$-th moment of $\vt$ for a $k>4$ is bounded, then we can tap into the main result of \cite{versh_cov_est} to show that the sample complexity can be reduced to from $O(n \log n)$ to $O(n (\log\log n)^2)$.

\section{Application to PCA in sparse data-dependent noise and its special cases}\label{apps} 

Consider the PCA in sparse data-dependent noise (PCA-SDDN) problem described earlier. In this case, $\yt = \lt + \wt$ with $\wt = \I_{\T_t} \M_{s,t} \lt$. Thus $\wt$ is sparse with support $\T_t$. 
The following is an easy corollary of Corollary \ref{only_cor}.

\begin{corollary}[PCA-SDDN]
Given data vectors $\yt := \lt + \I_{\T_t} \M_{s,t} \lt$, $t=1,2,\dots,\alpha$. 
Let $\Phat$ be the matrix of top $r$ eigenvectors of $\D:=\frac{1}{\alpha} \sum_t \yt \yt'$.
Assume that $\lt$ satisfies Assumption \ref{lt_mod}, $\max_t \|\M_{s,t} \P\|_2 \le q < 1$, the fraction of nonzeroes in any row of the noise matrix $[\w_1, \w_2, \dots, \w_\alpha]$ is bounded by $\bz$, and $\bz,q$ satisfy $3 \sqrt{\bz}q f + \epsbnd + \epsden < 1$. Here $\epsbnd,\epsden$ are as defined in Corollary \ref{only_cor}. Then, w.p. at least $1- 10n^{-10}$, $\SE(\Phat,\P) \le \frac{3 \sqrt{\bz}q f + \epsbnd}{1- (3 \sqrt{\bz}q f + \epsbnd + \epsden)}$.
\label{pca_sddn}
\end{corollary}

\begin{remark}[PCA-SDDN - alternate]
Another way to state Corollary \ref{pca_sddn} is as follows. Assume that $\lt$ satisfies Assumption \ref{lt_mod}, $\max_t \|\M_{s,t} \P\|_2 \le q < 1$, and the fraction of nonzeroes in any row of the noise matrix $[\w_1, \w_2, \dots, \w_\alpha]$ is bounded by $\bz$.
For an $\epsilon_\SE > 0$, if $\alpha \ge \alpha_0 = C \max\left( \frac{q^2 f^2}{\epsilon_\SE^2} (r \log n), f^2 (r + \log n) \right)$ and if $3 \sqrt{\bz}q f < 0.9 \frac{\epsilon_\SE}{1+\epsilon_\SE}$, then w.p. at least $1- 10n^{-10}$, $\SE(\Phat,\P) \le \epsilon_\SE$.
\label{pca_sddn_alt}
\end{remark}


It is also possible to solve the above problem using techniques from the sparse + low-rank matrix recovery (robust PCA) literature, e.g., \cite{rpca,robpca_nonconvex,rpca_gd, rmc_gd}. These also do not assume anything about whether the noise (sparse outlier) depends on the true data or not and hence do allow data-dependent noise. However, there are some differences.
(1) Our guarantee for PCA via SVD (and, in fact, all guarantees for PCA) {\em assume} that the noise is smaller than than the data (needs $q < 1$) where as robust PCA solutions are designed to handle noise (sparse outliers) that can have any magnitude. 
(2) Because of this, the robust PCA solutions are more expensive than the simple SVD solution that works for PCA. The most recent robust PCA solutions \cite{rpca_gd, rmc_gd} have nearly the same order of complexity as simple SVD, however in practice they are still slower.
(3) More importantly, all robust PCA via sparse + low-rank recovery solutions require the columns of $\P$ to be dense (not sparse). Their guarantees also require denseness of the right singular vectors. In our notation, these would be columns of the matrix $[\a_1, \a_2,  \dots, \a_\alpha]' \Lam^{-1}$.
These assumptions are necessary even for identifiability; otherwise the true data vector(s) may get wrongly classified as sparse outliers by the algorithm.
If the goal is only PCA (and not recovering the low rank matrix), then it is possible that the denseness assumption on right singular vectors can be removed, see e.g. \cite{outlier_pursuit}. 
Our guarantee given above for simple SVD does not require denseness of even the columns of $\P$. As shown in  \cite[Table 1]{corpca_nips}, when $\P$ contains sparse vectors, robust PCA solutions also fail in practice, while SVD does not (as long as the sparse noise is small of course).

The PCA-SDDN model given above is often a valid one for video analytics applications, where $\lt$ is the background layer of image frame $t$, $\T_t$ is the foreground (e.g., occlusion) support of frame $t$, and $\w_t:=\I_{\T_t} \M_{s,t} \lt$ is the difference between foreground and background intensities on $\T_t$. The above corollary is useful in problems involving subspace learning of slow changing videos (well modeled as being low rank) when the video is corrupted by foreground occlusions whose intensity is very similar to that of the background and that are correlated with the background (resulting in small magnitude and data-dependent $\wt$). Occlusions due to shadows often fall in this category. Another application is in using functional MRI (fMRI) data to learn the low-dimensional subspace in which resting state fMRI data lies. Even in the absence of external stimuli, it is well understood that the human brain is never fully resting. The sparse and small magnitude activations generated by random thoughts in the so-called ``resting state" brain are well modeled as $\wt$ described above. Besides these examples, we discuss two other practically relevant special cases of PCA-SDDN in the next two subsections.

\subsection{Special case: PCA with missing data}\label{apps_missing}
Let $\T_t$ denote the set of missing entries at time $t$. As explained earlier, if we set the missing entries to zero to define an $n$-length observed data vector $\yt$, then $\yt$ satisfies $\yt = \lt + \wt$ where  $\wt = - \I_{\T_t} \I_{\T_t}{}' \lt$ is the sparse ``error" or ``noise" due to the missing part of the data. In this case, $\M_{s,t} = - \I_{\T_t}{}'$. The bound on $q$ thus translates to a denseness assumption on the columns of $\P$.
Let $\mu$ be the densenesss (incoherence) parameter \cite{rpca} for $\P$, i.e., let $\mu$ be the smallest real number so that
\[
\max_i \|\I_i{}'\P\|_2^2 \le \mu^2 r / n.
\]
Also let $s:= \max_t |\T_t|$ be an upper bound on the number of missing entries at any time. 
It is easy to see that $\max_t ||\M_{s,t} \P||_2^2 = \max_t \|\I_{\T_t}{}' \P\|_2^2 \le s \max_{i=1,2,\dots,n}\|\I_i{}' \P\|_2^2 \le \mu^2 r s/n \equiv q^2$. 
We have the following corollary.  
\begin{corollary}[PCA with missing data]
Given data vectors $\yt := \lt - \I_{\T_t} \I_{\T_t}{}' \lt$, $t=1,2,\dots, \alpha$ with $|\T_t|\le s$.
%
Assume that $\lt$ satisfies Assumption \ref{lt_mod},
the fraction of missing entries in any row of the data matrix is at most $\bz$ where $\bz$ satisfies $3\sqrt{\bz} q f + \epsbnd  + \epsden < 1$ and  $q = \sqrt{\mu^2 r s/n}$.
Here $\epsbnd,\epsden$ are defined in Corollary \ref{only_cor}.
 Then, w.p. at least $1- 10n^{-10}$, $\SE(\Phat,\P) \le \frac{3 \sqrt{\bz}q f + \epsbnd}{1- 3 \sqrt{\bz}q f - \epsbnd - \epsden}$.
\label{cor_pca_miss}
\end{corollary}


Another way to solve the above problem would be to use techniques from the low-rank matrix completion (LRMC) literature to first complete the incomplete low-rank true data matrix $\L$ and then compute its left singular vectors. (1) This will need fewer observed entries because it relies on  more expensive techniques that are designed to actually deal with missing data instead of just SVD which treats the missing data as noise. However, the LRMC methods will also be much slower. In fact, SVD is often the initialization step for iterative LRMC solutions, e.g., \cite{lowrank_altmin}.
(2) It is hard to directly compare Corollary \ref{cor_pca_miss} with the LRMC guarantees since both use different assumptions, but the following can be said. LRMC results assume that observed entries are selected uniformly at random (or via a Bernoulli model), while Corollary \ref{cor_pca_miss} assumes a bound on the number of missing entries per row ($\bz$) and per column ($s/n$). 
(3) Moreover, LRMC results  require denseness of left and right singular vectors, while our Corollary \ref{cor_pca_miss} only needs denseness of left singular vectors.%

\subsection{Special case: subspace update step of a dynamic robust PCA solution} \label{apps_dynrpca}

A very important special case of PCA-SDDN occurs in the subspace update step of the Recursive Projected Compressive Sensing (ReProCS) approach to dynamic robust PCA \cite{rrpcp_perf,rrpcp_aistats}. For dynamic robust PCA, the observed data vector $\mt$ satisfies $\mt:= \lt + \xt$, where $\xt$ is a sparse outlier with support denoted by $\T_t$ at time $t$, and $\lt: = \P_{t} \at$ is the true data vector that lies in a low dimensional subspace that is ``slowly" changing; the subspace is fixed for a while and then changes by a little. To be precise we assume that $\P_{t} = \P_{t_j}$ for $t \in [t_j, t_{j+1})$ with $\SE( \P_{t_{j-1}}, \P_{t_j}) \le b_P \ll 1$ and with $t_{j+1}-t_j$ is lower bounded. One generative model for this is given in \cite[Equation (2)]{rrpcp_dynrpca}.
The goal is to track this changing subspace over time. The columns of the matrices $\P_{t_j}$'s are assumed to be dense (not sparse).
For simplicity we often use $\P_j := \P_{t_j}$.

ReProCS proceeds as follows. Given an accurate estimate of the previous subspace, denoted $\Phat_{t-1}$, it first projects $\mt$ orthogonal to $\Phat_{t-1}$ to get $\tilde\mt:=(\I - \Phat_{t-1}\Phat_{t-1}{}')\mt$. Because of the slow subspace change assumption, it can be argued that this nullifies most of $\lt$ and gives projected measurements of $\xt$. The problem of recovering $\xt$ from $\tilde\mt$ is now a standard Compressive Sensing (CS) problem \cite{candes_rip} in small noise, $\bm\beta_t:= (\I - \Phat_{t-1} \Phat_{t-1}{}') \lt$. ReProCS uses ell-1 minimization followed by support estimation and Least Squares based debiasing to solve this CS problem. Once $\xt$ is recovered, it recovers $\lt$ by subtraction, $\lhat_t = \mt - \xhat_t$. The estimates $\lhat_t$ are used to update the subspace estimate every $\alpha$ frames by solving either a PCA or an incremental PCA problem. It is assumed that the subspace change is slow enough so that $t_{j+1}-t_j > K \alpha$. This allows the subspace to be updated $K$ times, each time with a new set of $\alpha$ frames of $\lhat_t$, before it changes. The intuitive reason why this works is, after each subspace update, $\bm\beta_t$ reduces and hence the CS step error reduces. Thus, the error in $\lhat_t$ reduces and this, in turn, helps reduce the subspace recovery error at the next update.

To understand in a simple fashion how PCA-SDDN fits in here, assume that the subspace change is detected exactly at $t_j$ and the subspace update times are also aligned so that the first subspace update is done at $t=t_j+\alpha-1$, the second is at $t=t_j+2\alpha-1$, and so on. Let $\Phat_{j,k}$ denote the updated subspace estimate after the $k$-th update with $\Phat_{j,0} = \Phat_{j-1}$.
Thus, in the interval $[t_j, t_j+\alpha)$, $\P_{t} = \P_{j} $ and $\Phat_{t-1} =\Phat_{j,0} = \Phat_{j-1}$. Assume that the previous subspace is recovered with $\epsilon$ error, i.e. $\SE(\Phat_{j-1},\P_{j-1}) \le \epsilon$ before $t_j$. Using this and the denseness of columns of $\P_{t}$, one can show that $\xt$ is recovered accurately for all $t \in [t_j, t_j+\alpha)$. The same analysis also shows that $\|\B_t\| \le 1.2$ where $\B_t:=  \I_{\T_t}{}' (\I - \Phat_{j,0} \Phat_{j,0}{}')  \I_{\T_t}$.
Then, using simple extra assumptions, one can argue that $\T_t$, which is the support of the outlier vector $\xt$, can be recovered exactly. With this, it can be shown that $\et:= \lhat_t - \lt = \xt - \xhat_t$ satisfies
\[
\et = \I_{\T_t} \B_t^{-1} \I_{\T_t}{}' (\I - \Phat_{j,0} \Phat_{j,0}{}') \l_t.
\]
Thus, the subspace update step is an instance of PCA-SDDN with $\yt \equiv \lhat_t$ and $\wt \equiv \et$ for $t = t_j, t_j+1, \dots t_j+\alpha-1$. We can apply Corollary \ref{pca_sddn} with $\bz$ being the maximum fraction of nonzeros in any row of $[\x_{t_j},\x_{t_j+1}, \dots, \x_{t_j+\alpha-1}]$ and with $\M_{s,t} =  \B_t^{-1} \I_{T_t}{}' (\I - \Phat_{j,0} \Phat_{j,0}{}')$.
Thus, $\|\M_{s,t} \P \|_2 \le 1.2 \SE(\Phat_{j,0}, \P_{j}) = 1.2 \SE(\Phat_{j-1},\P_j) \le 1.2 (\epsilon + b_P):= q_0$.
Apply the PCA-SDDN result with $q=q_0$ and $\epsilon_\SE = q_0/4$. Thus, if $3 \sqrt{b} f < 0.2$, and if  $\alpha \ge \alpha_0 = C f^2 (r \log n)$, then, after the first subspace update at $t=t_j+\alpha-1$, $\SE(\Phat_{j,1},\P_j) \le q_0/4$. This, in turn ensures that the bound on the noise $\bm\beta_t$ seen by the CS steps for the next interval, $[t_j+\alpha, t_j+2\alpha-1)$, is significantly  smaller.
Using denseness and the bound on  $\SE(\Phat_{j,1},\P_j)$ and simple extra assumptions, one can then argue that the CS step gives a more accurate estimate of $\xt$ and that $\T_t$ is correctly recovered. Thus, for the interval $[t_j+\alpha, t_j+2\alpha)$, $\et$ again  the expression given above but with $\Phat_{j,0}$ replaced by $\Phat_{j,1}$. 
Thus, for the second PCA update at $t=t_j+2\alpha-1$, $\|\M_{s,t} \P \|_2 \le 1.2  \SE(\Phat_{j,1}, \P_{j}) \le 1.2 q_0/4  := q_1$. Apply the PCA-SDDN result with $q= q_1$ and with $\epsilon_\SE = q_1/4 = 1.2 q_0/4^2$ to show that $\SE(\Phat_{j,2},\P_j) \le q_1/4 = 1.2 q_0/4^2$. Repeating this process, after $K$ updates, $\SE(\Phat_{j,K},\P_j) \le (1.2/4)^{K-1} 0.25 q_0$. By picking $K$ large enough, we ensure that this bound is below $\epsilon$. We set the final estimate $\Phat_{j} : = \Phat_{j,K}$. This serves as the starting point for estimating the next change.%


In the above discussion, to explain things simply, in each subspace update, we used simple SVD. However, if we assume that only one (or only a few) direction(s) change at each subspace update time as was done in \cite{rrpcp_dynrpca}, one can get an improved result by first accurately estimating the new direction(s) that got added to the subspace by solving a problem of PCA with partial subspace knowledge repeated $K$ times. Finally, when $q$ is small enough (is below $2 \epsilon$), one can re-estimate the entire subspace by solving a standard PCA problem. This is done to delete the removed direction(s) from the subspace estimate. For the former problem, we replace standard SVD by a projection-SVD step. To analyze it, we develop a modification of the ideas from this work to prove a result for PCA with partial subspace knowledge when noise is data dependent, see \cite[Theorem 6.4]{rrpcp_dynrpca}. For the deletion, step we use the PCA-SDDN result given above.


\section{Numerical Experiments} \label{expts}

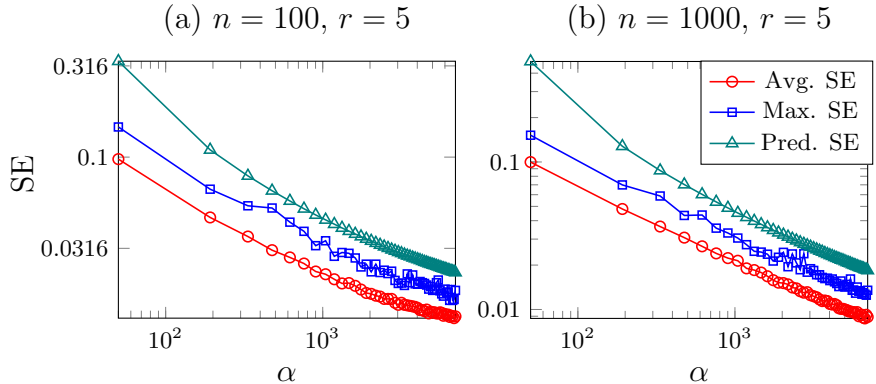
\begin{figure}[ht!]
%
\begin{center}
\begin{tikzpicture}
    \begin{groupplot}[
        group style={
            group size=2 by 1,
        },
        width=0.5\textwidth,
	    height=5cm,
		my legend style compare,
    	]
        \nextgroupplot[
	        xlabel={\large{$\alpha$}},
			ylabel={\large{SE}},
		    ymode=log,
		    xmode=log,
            log y ticks with fixed point,
            enlargelimits=false,
            title={\large{(a) $n=100$, $r=5$}},
                x label style={at={(axis description cs:0.5,0)},anchor=north},
                y label style={at={(axis description cs:-.01,0.45)},anchor=west},
        ]
        \node [text width=1em,anchor=north west] at (rel axis cs: 0.45,-1.2)
                {\subcaption{\label{fig:sebnd_smallnr}}};
	\addplot table[x index = {0}, y index = {1}]{\SEbound};
	\addplot table[x index = {0}, y index = {2}]{\SEbound};
	\addplot table[x index = {0}, y index = {3}]{\SEbound};
        \nextgroupplot[
        		legend entries={
            	Avg. SE,
            	Max. SE,
            	Pred. SE,
            },
			legend style={
                at={([yshift=-5pt]0.504,1.055)},
                anchor=north west,
            },
        xlabel={\large{$\alpha$}},
		            ymode=log,
		            xmode=log,
                    log y ticks with fixed point,
            enlargelimits=false,
                title={\large{(b) $n=1000$, $r=5$}},
                x label style={at={(axis description cs:0.5,0)},anchor=north},
        ]
\node [text width=1em,anchor=north west] at (rel axis cs: 0.45,-1.2)
                {\subcaption{\label{fig:sebnd_largenr}}};
	\addplot table[x index = {0}, y index = {1}]{\SEboundnew};
	\addplot table[x index = {0}, y index = {2}]{\SEboundnew};
	\addplot table[x index = {0}, y index = {3}]{\SEboundnew};
        \end{groupplot}
\end{tikzpicture}
\end{center}
\vspace{-0.5cm}
\caption{\small{Numerically computed mean and maximum values of $\SE(\Phat,\P)$ (over 100 trials) and its bound from Theorem \ref{mainthm}. The bound used $c=1$.}}
\vspace{-0.5cm}
\label{fig:SEbnd}
\end{figure}		

In our first experiment we numerically demonstrate the tightness of the bound of Theorem \ref{mainthm} by plotting the numerically computed subspace error and the bound suggested by the theorem. In the expressions for $\epsbnd$ and $\epsden$ in the bound, there is an unspecified constant $c$. We set $c=1$ while plotting the bound. We generated the data as $\lt = \P \at$, where $\P$ was generated by ortho-normalizing the columns of an $n \times r$ matrix with  independent identically distributed (iid) standard Gaussian entries.  We generated the coefficients $(\at)_i$ as iid $uniform(-6, 6)$. With this, $\lambda^+ = \lambda^- = 12$ and $f = 1$. We generated the uncorrelated noise as $\vt = \bm{B} \bm{c}_t$ where $\bm{B}$ is generated by orthonormalizing the columns of an $n \times r_v$ matrix with iid standard Gaussian entries and $(\bm{c}_t) \sim $\emph{unif}$(-q_i, q_i)$ with $q_i = 1.1 - 0.1i/r_v$. The data-dependent noise was generated as $\wt = \I_{\T_t} \M_{s,t} \frac{q}{\|\M_{s,t} \P\|} \lt$ and each entry of $\M_{s,t}$ was generated independently as the absolute value of a standard Gaussian r.v. (taking the absolute value ensures that $ \E[\M_{s,t}] \neq 0$). Further, $\T_t$ was generated to follow \cite[Model D.24]{rrpcp_dynrpca} with $s=5$, $\rho=1$ and $b_0 = 0.05$ (simulates a 1D moving object that moves every so often). We set $\yt = \lt + \wt + \vt$.
We used $r_v =r$, $q=0.001$. From the support change model, $\bz = b_0= 0.05$.
We varied $\alpha$ in the range of $[29, 7000]$ and computed $\Phat$ and $\SE(\Phat,\P)$ for each value of $\alpha$.  Fig \ref{fig:sebnd_largenr} used $n=1000$ and $r=10$ while Fig. \ref{fig:sebnd_smallnr} used $n = 100$, $r = 5$. We show the mean and maximum values of the numerically computed $\SE(\Phat,\P)$, and the bound predicted by Theorem \ref{mainthm}, as a function of $\alpha$ in Fig. \ref{fig:SEbnd}. The mean and max are computed over 100 Monte Carlo trials.
Notice that the bound appears quite tight in both figures. 

\pgfplotsset{every axis title/.append style={at={(.5,0.95)}}}
\begin{figure}[ht!]
\begin{center}
\begin{tikzpicture}
    \begin{groupplot}[
        group style={
            group size=2 by 2,
            vertical sep=1.35cm,
            horizontal sep=1.2cm
        },
        width = .5\textwidth,
        height = 4cm
    	]
        \nextgroupplot[
        		view={0}{90},
                xlabel=\large{$r$},
                ylabel=\large{$\alpha$},
                colormap/blackwhite,
                title={\large{(a)}},
                x label style={at={(axis description cs:0.5,-0.13)},anchor=north},
        ]
        \node [text width=1em,anchor=north west] at (rel axis cs: 0,1)
                {\subcaption{\label{fig:phasetransvsr}}};
			    \addplot3[surf] file{figures_new_pca/PhaseTransvsr_bounded.dat};	

        \nextgroupplot[
		        view={0}{90},
                xlabel=\large{$n$},
                colormap/blackwhite,
                title={\large{(b)}},
				x label style={at={(axis description cs:0.5,-0.13)},anchor=north},
        ]
        \node [text width=1em,anchor=north west] at (rel axis cs: 0,1)
                {\subcaption{\label{fig:phasetransvsnbnd}}};

                	\addplot3[surf] file {figures_new_pca/PhaseTransvsn_bounded.dat};
                			
		            \nextgroupplot[
		       view={0}{90},
               xlabel=\large{$n$},
               ylabel=\large{$\alpha$},
               colormap/blackwhite,
               title={\large{(c)}},
               scaled y ticks=false, tick label style={/pgf/number format/fixed},
               x label style={at={(axis description cs:0.5,-0.13)},anchor=north},
        ]
    \node [text width=1em,anchor=north west] at (rel axis cs: 0,1)
                {\subcaption{\label{fig:phasetransvsngausr}}};

        	\addplot3[surf] file {figures_new_pca/PhaseTransvsn_gaussian_rv_r.dat};
        			            \nextgroupplot[
		       view={0}{90},
               xlabel=\large{$n$},
               colormap/blackwhite,
               title={\large{(d)}},
               scaled ticks=false, tick label style={/pgf/number format/fixed},
               x label style={at={(axis description cs:0.5,-0.13)},anchor=north},
        ]
\addplot3[surf] file {figures_new_pca/PhaseTransvsn_gaussian.dat};
        \node [text width=1em,anchor=north west] at (rel axis cs: 0,0)
                {\subcaption{\label{fig:phasetransvsngausn}}};
%
        	
        \end{groupplot}
\end{tikzpicture}
\end{center}
\vspace{-0.5cm}
\caption{\small{
The grey scale intensity represents the numerically computed probability (fraction of times) that $\SE(\Phat,\P) \le \epsilon$ for the $\epsilon$ value given in the text. Black is zero and white is one. Fig. \ref{fig:phasetransvsr}: displays the probability for $r$ versus $\alpha$ for $n = 100$, $r_v = r$ and bounded data and noise model; Fig. \ref{fig:phasetransvsnbnd}: $n$ versus $\alpha$ for $r_v=r = 1$ and bounded data and noise; Fig. \ref{fig:phasetransvsngausr}: $n$ versus $\alpha$ for $r_v =r= 1$ and Gaussian data and noise; Fig. \ref{fig:phasetransvsngausn}: $n$ versus $\alpha$ for $r = 1$, $r_v = n$ and Gaussian.%
}}
\vspace{-0.5cm}
\label{fig:phasetrans}
\end{figure}

In our second experiment, we use Monte Carlo to estimate the probability of $\SE(\Phat,\P) \le \epsilon$ for various values of $r$ and $\alpha$ with $n$ fixed or various values of $n$ and $\alpha$ with $r$ fixed. All our estimates used 100 Monte Carlo trails.
The probability is displayed in Fig. \ref{fig:phasetrans} as a grey scale intensity with black denoting zero and white denoting one. This helps to numerically compute the value of $\alpha$ needed for a given $n,r$ to ensure that the probability is close to one (smallest $\alpha$ for which the color is white).
Fig. \ref{fig:phasetransvsr} varies $r$ and $\alpha$ for $n=100$. All other parameters were the same as in the first experiment. So $n=100$, $\bz=0.05$, $q=0.001$, $f=1$, $\lambda^-=12$,  $r_v=r$, $\lambda_v^+ = 1.1$.
To generate the plot we used $\epsilon = 1.5 \left( \sqrt{\bz} (2q + q^2)f  + \frac{\lambda_{v, \bm{P}, \bm{P}_{\perp}} / \lambda^-}{1 - \frac{\lambda_{v, \rest^+} - \lambda_{v, \bm{P}}^-}{\lambda^-}} \right)$. As can be seen, the dependence of $\alpha$ on $r$ is linear. This matches what our guarantees claim about the sample complexity: $\alpha$ needs to be $C \max(r_v,r)\log n$. Here $r_v=r$.

In the other three sub-figures we fix $r$ and $r_v$ and evaluate the dependence of $\alpha$ on $n$ in various settings. In Fig. \ref{fig:phasetransvsnbnd}, we set $r_v=r$, $r=1$, and other parameters were as above. Thus both true data and noise are bounded. We display the numerically estimated probability of $\SE(\Phat,\P) \le \epsilon$ for various values of $n$ and $\alpha$. Recall that for bounded data and noise, the required $\alpha$ is proportional to $\max(r_v,r)\log n$, i.e., it depends logarithmically on $n$. Logarithmic variation is hard to observe numerically unless a very large range of $n$ is used. This is also what is seen from Fig. \ref{fig:phasetransvsnbnd}. For the range of values of $n$, the required $\alpha$ seems nearly constant.

For Fig. \ref{fig:phasetransvsngausr}, we replaced the boundedness assumption by a Gaussian assumption on data and noise. We still generated $\vt = \bm{B} \bm{c}_t$ where $\B$ is an $n \times r_v$ matrix generated as before and we set $r_v=r$. But now we generated $(\bm{c}_t)_i  \iidsim \mathcal{N}(0, q_i^2)$ with $q_i = 0.9 - 0.4i/r_v$.  We generated $\lt = \P \at$ where $\at \iidsim \n(0,100)$. Here $\n(0,\sigma^2)$ refers to a zero Gaussian distribution with variance $\sigma^2$.
Thus $\Lam=100$ and $f=1$. Everything else was the same as in the first experiment, thus $\wt = \I_{\T_t} \M_{s,t} \lt$ with $\T_t$ and $\M_{s,t}$ were generated as described earlier. Thus, $\bz=0.05, q=0.001$. Also, $\lambda^+=100 = \lambda^-$ and $f=1$ and $\lambda_v^+=0.9$. We fixed $r_v=r=1$ and vary $n$ and $\alpha$ as was also done in Fig. \ref{fig:phasetransvsnbnd}. Notice now that the required $\alpha$ to achieve high-enough (white) probability of success does increase with $n$. It is hard to say though whether the dependence is indeed linear as predicted by our theorem. As we pointed out earlier in Remark \ref{tighten_subg}, for this setting since $r_v=r \ll n$, 
it may be possible to tighten the required sample complexity lower bound.

Finally, for Fig. \ref{fig:phasetransvsngausn}, we generated data exactly as for Fig. \ref{fig:phasetransvsngausr} but with $r_v=n$ (instead of $r_v=r$). As can be seen, now the required sample complexity does indeed increase linearly with $n$. This matches what is predicted by our main result for the $r_v=n$ case.
Notice that in the last two sub-figures where we used Gaussian noise and data, we have increased signal power as compared to the bounded case. If this was not increased, the required $\alpha$ to achieve large enough probability of success would be very large and would lead to very slow computations. All experiments used the MATLAB command svds for computing $\Phat$. All codes are available at \url{https://github.com/praneethmurthy/correlated-pca}.

\section{Conclusions and Future Work} \label{conclude}
In this work, we studied the PCA problem when the noise can be non-isotropic and/or data-dependent, and as a result, in general the data and noise are correlated. We obtained guarantees under both a bounded-ness assumption and a sub-Gaussian assumption on the data and noise.
When the uncorrelated noise has effective dimension $O(r)$, under the bounded-ness assumption, a simple assumption on data-noise correlation, and a bound on the ratio between noise power and minimum signal space eigenvalue, we showed that the required sample complexity for PCA is near optimal. Under the sub-Gaussian assumption, the required sample complexity as predicted by our results increases to $O(n)$ which is comparable to what existing results for isotropic Gaussian noise also need. However, as noted in Remark \ref{tighten_subg}, in the setting where the sub-Gaussian noise has effective dimension $r_v \ll n$, it should be possible to tighten this.


The result given here assumes that the $\lt$'s are mutually independent random variables.
Mutual independence can be replaced by an autoregressive (AR) model on the $\lt$'s. As long as the AR parameter is not too large, it should be possible to get a result very similar to the one given in this work using the matrix Freedman's inequality \cite{freedman} or a little weaker than the one given here using matrix Azuma \cite{tail_bound}. The latter would generalize the approach developed in \cite{rrpcp_aistats} to for analyzing the subspace update step of ReProCS under an AR model on the $\lt$'s.

In ongoing work, we are studying the problem of PCA in data-dependent noise when partial knowledge of the subspace is available and its implications for the subspace update step of ReProCS \cite{rrpcp_dynrpca}. 
A useful open question for future work is how to analyze algorithms for streaming PCA, e.g., the block-stochastic power method, in the data-dependent noise setting. This was studied in \cite{onlinePCA3} under the spiked covariance model, or in \cite{streamingPCA2} for an arbitrary observed data covariance matrix, but for $r=1$ dimensional PCA.



\appendix

\renewcommand{\thetheorem}{\thesection.\arabic{theorem}}

\section{Proof of Main Result: Bounded Case}\label{proof_mainthm}


\begin{proof}[Proof of Theorem \ref{mainthm} under Assumption \ref{lt_mod}] 
Apply the Davis-Kahan $\sin \theta$ theorem \cite{davis_kahan} summarized in Lemma \ref{sintheta} 
with
\[
\D_0 = \frac{1}{\alpha} \sum_t \lt \lt' + \P\P'\bm\Sigma_v \P\P' =  \P( \frac{1}{\alpha} \sum_t \at \at' + \P'\bm\Sigma_v \P ) \P'.
\]
Observe that $\lambda_{r+1}(\D_0) = 0$ and, using Weyl and $\E[\frac{1}{\alpha} \sum_t \at \at'] = \Lam$,
\[
\lambda_{r}(\D_0) \ge \lambda_{\min}(\P'\bm\Sigma_v \P) + \lambda_{\min}(\frac{1}{\alpha} \sum_t \at \at') \ge \lambda_{v,\P}^- + \lambda^- - \|\frac{1}{\alpha} \sum_t \at \at' - \Lam\|_2 .
\]
Thus, rewriting $\D - \D_0$ as $ \E[\D - \D_0] + (\D - \D_0 - \E[\D - \D_0])$, 
\begin{align*}
\SE(\Phat,\P) \le  \frac{\|(\E[\D-\D_0])\P\|_2 + \|\D - \D_0 - \E[\D - \D_0]\|_2}{ \lambda^- + \lambda_{v,\P}^- - \|\frac{1}{\alpha} \sum_t \at \at' - \Lam\|_2   - \lambda_{\max}(\E[\D-\D_0]) -  \|\D - \D_0 - \E[\D - \D_0]\|_2}
\end{align*}
Observe that
\begin{align*}
\D - \D_0 = & \frac{1}{\alpha} \sum_t \vt \vt' +  \frac{1}{\alpha} \sum_t \lt \vt' + \frac{1}{\alpha} \sum_t \vt \lt' \\
& + \frac{1}{\alpha} \sum_t \lt \wt' + \frac{1}{\alpha} \sum_t \wt \lt' + \frac{1}{\alpha} \sum_t \wt \wt'  -\P\P'\bm\Sigma_v \P\P'
\end{align*}
Since $\vt$ is uncorrelated with $\lt$, the expected value of the second and third terms is zero. Thus,
\begin{align}
\E[\D - \D_0]  = &  (\bm\Sigma_v - \P\P'\bm\Sigma_v \P\P') +   \frac{1}{\alpha} \sum_t \P \Lam \P'\M_t{}' + \frac{1}{\alpha} \sum_t \M_t \P \Lam \P' + \frac{1}{\alpha} \sum_t \M_t \P \Lam \P'\M_t{}'
\label{ED_D0_dif}
\end{align}
The third term on the RHS is a transpose of the second one. For the second term, write $\M_t = \M_{2,t} \M_{1,t}$ and then use Cauchy Schwartz for sums of matrices (see e.g. \cite[Lemma A.6]{rrpcp_aistats}) to conclude that
\begin{align}
\| \frac{1}{\alpha} \sum_t \P \Lam \P'\M_{1,t}{}'\M_{2,t}{}' \|_2^2 & \le  \| \frac{1}{\alpha} \sum_t \P \Lam \P'\M_{1,t}{}'\M_{1,t}\P \Lam \P' \|_2  \| \frac{1}{\alpha} \sum_t \M_{2,t} \M_{2,t}{}' \|_2 \nn \\
& \le \max_t \|\M_{1,t}\P \Lam \P'\|_2^2 \ \bz \le  (q \lambda^+)^2 \bz.
\label{bnd_avg_sig_noise_cor_bnd_1}
\end{align}
The second inequality used \eqref{M2t_bnd} and the third inequality used $\|\M_{1,t} \P\|\le q$.
The fourth term is handled similarly:
\begin{align}
\| \frac{1}{\alpha} \sum_t \M_t \P \Lam \P'\M_{1,t}{}' \M_{2,t}{}'\|_2^2 & \le  \| \frac{1}{\alpha} \sum_t \M_{t} \P \Lam \P'\M_{1,t}{}'\M_{1,t}\P \Lam \P'\M_{t}{}' \|_2  \| \frac{1}{\alpha} \sum_t \M_{2,t} \M_{2,t}{}' \|_2 \nn \\
& \le  \max_t \|\M_{t}\P \Lam \P'\M_{1,t}{}' \|_2^2  \ \bz \le  (q^2 \lambda^+)^2 \bz.
\label{bnd_avg_sig_noise_cor_bnd_2}
\end{align}
Since $(\P \P' + \P_\perp \P_\perp{}') = \I$, we can always rewrite $\bm\Sigma_v =(\P \P' + \P_\perp \P_\perp{}') \bm\Sigma_v (\P \P' + \P_\perp \P_\perp{}')$.
Using this and the last three equations above, 
\begin{align*}
\|(\E[\D-\D_0])\P \|_2
& \le \|\P_\perp{}'\bm\Sigma_v \P\|_2 + 2\sqrt{\bz} q \lambda^+ + \sqrt{\bz} q^2 \lambda^+  = \lambda_{v,\P,\P_\perp} + \sqrt{\bz} (2q + q^2) \lambda^+
\end{align*}
and
\begin{align*}
 \lambda_{\max}(\E[\D-\D_0]) & \le \lambda_{v,\rest}^+ + \sqrt{\bz} (2q + q^2) \lambda^+ .
\end{align*}
Thus, by Lemma \ref{sintheta},
\begin{align}
\SE(\Phat,\P) & \le   
%
%
 \frac{\frac{\lambda_{v,\P,\P_\perp}}{\lambda^-} + \sqrt{\bz} (2q + q^2) f  + \frac{\|\D-\D_0 - \E[\D-\D_0]\|_2}{\lambda^-} }{1 -   ( \frac{\lambda_{v,\rest}^+ - \lambda_{v,\P}^-}{\lambda^-} + \sqrt{\bz} (2q + q^2) f) - \frac{ \|\frac{1}{\alpha} \sum_t \at \at' - \Lam\|_2}{\lambda^-} - \frac{\|\D-\D_0 - \E[\D-\D_0]\|_2}{\lambda^-} }
 \label{fin_bnd}
\end{align}

To bound $\|\D-\D_0 - \E[\D-\D_0]\|_2$ and $\|\frac{1}{\alpha} \sum_t \at \at' - \Lam\|_2$, we use concentration bounds from the following lemma which we prove in Appendix \ref{proof_hp_bnds}.
\begin{lem}
\label{hp_bnds}
With probability at least $1 - 10n^{-10}$, if $\alpha^3 > \max(r_v,r) \log n$, then, 
\begin{align*}
&  \left\|\frac{1}{\alpha}\sum_t \at \at{}' -  \Lam \right\| \le    c {\eta} f \sqrt{\frac{r + \log n}{\alpha}}  \lambda^-, \\
& \left\|\frac{1}{\alpha}\sum_t \lt \wt{}' - \frac{1}{\alpha}\E[\sum_t \lt \wt{}'] \right\|_2 \le c\sqrt{\eta} q f  \sqrt{\frac{r \log n}{\alpha}}     \lambda^-, \\ 
& \left\|\frac{1}{\alpha}\sum_t \wt \wt{}' - \frac{1}{\alpha}\E[\sum_t \wt \wt{}'] \right\|_2 \le  c\sqrt{\eta} q^2  f  \sqrt{\frac{r \log n}{\alpha}}   \lambda^-, \\ 
& \left\|\frac{1}{{\alpha}} \sum_t \lt {\vt}' \right\|_2 \le  c\sqrt{\eta} \sqrt{\frac{\lambda_v^+}{\lambda^-}f} \sqrt{\frac{ \max(r_v, r) \log n}{\alpha}} \lambda^-, \\
& \left\| \frac{1}{{\alpha}} \sum_t  \vt {\vt}' - \frac{1}{{\alpha}} \E[ \sum_t  \vt {\vt}' ] \right\|_2 \le  c\sqrt{\eta}  \frac{\lambda_v^+}{\lambda^-} \sqrt{\frac{r_v \log n}{\alpha}} \lambda^-.
\end{align*}
\end{lem}
Thus, w.p. $\ge 1- 10n^{-10}$,
$\left\|\frac{1}{\alpha}\sum_t \at \at{}' -  \Lam \right\| \le  c {\eta} f \sqrt{\frac{r + \log n}{\alpha}}  \lambda^-
$
and
$
\|\D-\D_0 - \E[\D-\D_0]\|_2 \le  \epsbnd_0 \lambda^-
$
where
\begin{align}
\epsbnd_0 := c \sqrt{\eta}
 \max \left(
q f \sqrt{\frac{r \log n}{\alpha}} ,  \sqrt{\frac{\lambda_v^+}{\lambda^-} f} \sqrt{\frac{ \max(r_v, r) \log n}{\alpha}}, \frac{\lambda_v^+}{\lambda^-} \sqrt{\frac{r_v \log n}{\alpha}}
\right)
\label{def_epsbnd}
\end{align}
Clearly, $\epsbnd_0 \le \epsbnd$ defined in the statement of Theorem \ref{mainthm}.  Combining these bounds with \eqref{fin_bnd}  we get our final result.
\end{proof}

\section{Proof of Concentration Bounds}\label{proof_hp_bnds}

\begin{proof}[Proof of Lemma \ref{hp_bnds}]  \

{\em $\at \at'$ term. }
Using  Vershynin's sub-Gaussian result (Theorem 5.39 of \cite{vershynin}) applied to $\frac{1}{\alpha} \sum_t \a_t \a_t{}'$, and using the fact that the $\at$'s are $r$-length independent sub-Gaussian vectors with sub-Gaussian norm bounded by $\sqrt{\eta \lambda^+}$, we get the following: with probability at least
$1 - 2\exp\left(r \log 9 - \alpha \frac{c (\epsilon_1 \lambda^-)^2 }{(4 \eta \lambda^+)^2}\right) = 1 - 2\exp\left(r \log 9 - \alpha \frac{c \epsilon_1^2 }{16 \eta^2 f^2}\right)$,
\[
\|\frac{1}{\alpha} \sum_t \at \at{}' - \Lam\|_2 \le \epsilon_1 \lambda^-
\]
Set $\epsilon_1 = c \eta f \sqrt{\frac{r + 11\log n}{\alpha}}$. Then, the above event holds w.p. at least $1 - 2n^{-10}$ 
%

{\em $\lt \wt'$ term. } This and all other items use Matrix Bernstein for rectangular matrices, Theorem 1.6 of \cite{tail_bound}. This says the following. For a finite sequence of $d_1 \times d_2$ zero mean independent matrices $\Z_k$ with
\begin{align*}
& \|\Z_k\|_2 \le R, \text{ and}  \max(\|\sum_k \E[ \Z_k{}'\Z_k ]\|_2, \|\sum_k \E[ \Z_k\Z_k{}' ]\|_2) \le \sigma^2, 
\end{align*}
we have $\Pr(\|\sum_k \Z_k\|_2 \ge s ) \le (d_1+d_2) \exp\left(- \frac{s^2/2}{\sigma^2 + Rs / 3}\right)$.

Let $\Z_t := \lt \wt{}'$. We apply this result to $\tilde{\Z}_t:= \Z_t - \E[\Z_t]$ with $s = \epsilon \alpha$. To get the values of $R$ and $\sigma^2$ in a simple fashion, we use the facts that (i)
if $\|\Z_t\|_2 \le R_1$, then $\|\tilde{\Z}_t\| \le 2R_1$; and (ii) $\sum_t \E[\tilde{\Z}_t \tilde{\Z}_t{}'] \preccurlyeq \sum_t \E[\Z_t \Z_t{}']$. 
Thus, we can set $R$ to two times the bound on $\|\Z_t\|_2$ and  we can set $\sigma^2$ as the maximum of the bounds on $\|\sum_t \E[\Z_t \Z_t{}']\|_2$ and $\|\sum_t \E[\Z_t{}' \Z_t]\|_2$.

It is easy to see that  $R = 2 \sqrt{\eta  r \lambda^+} \sqrt{\eta r  q^2 \lambda^+} = 2 \eta r q \lambda^+$. To get $\sigma^2$, observe that
\begin{align*}
\left\| \sum_t \E[\wt \lt{}' \lt \wt{}'] \right\|_2 & \le \alpha (\max_{\lt} \|\lt\|^2) \cdot  \|\E[\wt \wt{}']\| \\
& \le \alpha \eta  r \lambda^+ \cdot  q^2 \lambda^+  = \alpha \eta r q^2 (\lambda^+)^2.
\end{align*}
Repeating the above steps, we get the same bound on $\| \sum_t \E[\Z_t \Z_t{}'] \|_2$. Thus,  $\sigma^2 = \alpha \eta r q^2 (\lambda^+)^2$.

Thus, we conclude that,
\begin{align}
\left\|\sum_t \lt \wt{}' - \E[\sum_t \lt \wt{}'] \right\|_2 \ge \epsilon \alpha
\label{lt_wt_conc}
\end{align}
w.p. at most
$2n \exp\left(- \frac{\epsilon^2 \alpha/2}{ \eta r q^2 (\lambda^+)^2  + (\eta r q \lambda^+ \epsilon / 3 \alpha)}\right).
$

Set $\epsilon = \epsilon_0 \lambda^-$ with $\epsilon_0 = c \sqrt{\eta} q f \sqrt{\frac{r \log n}{\alpha}}$. If $\alpha^3 >  \eta (r \log n)$, then clearly, $R \epsilon \le \sigma^2$.
With this $\epsilon$, if $\alpha^3 >  (r \log n)$, \eqref{lt_wt_conc} holds w.p. at least $1-2 n^{-10}$.

{\em $\wt \wt{}'$ term. }
We again apply  matrix Bernstein and proceed as above. In this case, $R = 2 \eta r q^2 \lambda^+$ and $\sigma^2 = \alpha \sigma_1^2$, $\sigma_1^2 = \eta r q^4 (\lambda^+)^2$. Thus $R < \sigma^2 / 2q^2$.
%
Set $\epsilon = \epsilon_2 \lambda^-$ with $\epsilon_2 = c \sqrt{\eta} q^2 f \sqrt{\frac{r \log n}{\alpha}} < 1$. Then, can again show that if $\alpha^3 > (r \log n)$, the probability of the bad event is bounded by $2n^{-10}$.


{\em $\lt \vt'$ and $\vt \vt'$ terms. } Apply matrix Bernstein as done above.
\end{proof}

\section{Proof of Main Result: Sub-Gaussian Case}\label{proof_mainthm_subg}


\begin{proof}[Proof of Theorem \ref{mainthm} under Assumption \ref{lt_mod_subg}]
The only thing that changes in this case is the $\epsilon$ values used for the various terms in Lemma \ref{hp_bnds}. These change because we cannot apply matrix Bernstein now. Instead, for all terms, we apply the following simple modification of Vershynin's sub-Gaussian result \cite[Theorem 5.39]{vershynin}. This lemma follows using exactly the proof approach of \cite[Theorem 5.39]{vershynin} but with the following change. If two  r.v.s $x$, $y$ are sub-Gaussian with sub-Gaussian norms $k_x$, $k_y$ respectively, then the r.v. $z:=xy$ is sub-exponential with sub-exponential norm bounded by $c k_x k_y$. This fact itself follows by Cauchy-Schwartz and the definitions of the sub-Gaussian and sub-exponential norms: $\E[|xy|^p]^{1/p} \le ((\E[|x|^{2p}] \E[|y|^{2p}])^{1/2})^{1/p} = (\E[|x|^{2p}] \E[|y|^{2p}])^{1/2p} \le ( (\sqrt{2p}k_x)^{2p}(\sqrt{2p}k_y)^{2p} )^{1/2p} =(\sqrt{2p}k_x)(\sqrt{2p}k_y) = p (2 k_x k_y)$. The first inequality used Cauchy-Schwartz, the second inequality used the definition of sub-Gaussian norm \cite[Sec. 5.2]{vershynin}.
\begin{lem}\label{lem:cross_subg}
Let $\bx_i$, $i=1,2,\dots,N$, and $\by_i, i=1,2,\dots, N$ be zero mean sub-Gaussian random vectors with sub-Gaussian norms bounded by $K_x$ and $K_y$ respectively. Each $\bx_i, \by_i$ is in $\Re^n$. Also $\{\bx_i, \by_i\}$, $i=1,2,\dots, N$ are mutually independent.
Then,
\begin{align*}
\Pr\left(\norm{\frac{1}{N} \sum_i \bx_i \by_i{}' - \E \left[\frac{1}{N} \sum_i \bx_i \by_i{}'\right]} \le t \right) \geq 1 - 2 \exp \left(n \log 9 - \frac{t^2 N}{16c(K_xK_y)} \right)
\end{align*}
\end{lem}

Using the above lemma, we can conclude the following.
\begin{enumerate}
\item Bounding $\norm{\frac{1}{\alpha} \sum_t \lt \wt{}' - \frac{1}{\alpha}\E[\sum_t \lt \wt{}']}_2$: Apply Lemma \ref{lem:cross_subg} with $\bx_t = \lt$, $\by_t = \wt$, $N \equiv \alpha$.  Then $K_x = c\sqrt{\lambda^+}$ and $K_y = c\sqrt{q^2\lambda^+}$, thus
     $K_x K_y = c q \lambda^+ $.
    Set $t = \epsilon_0 \lambda^-$  with $\epsilon_0 = c f \sqrt{\frac{n}{\alpha}}$. Then
\begin{align*}
\Pr\left(\norm{\frac{1}{\alpha} \sum_t \lt \wt{}' - \frac{1}{\alpha}\E[\sum_t \lt \wt{}']}_2  \le \epsilon_0 \lambda^-\right) \geq 1- 2\exp(-c n)
\end{align*}

\item Bounding $\norm{\frac{1}{\alpha} \sum_t \wt \wt{}' - \frac{1}{\alpha}\E[\sum_t \wt \wt{}']}_2$:  Proceed as above. Use $\epsilon = \epsilon_0 \lambda^-$.

\item Bounding $\norm{\frac{1}{\alpha} \sum_t \at \at{}' -  \Lam }_2$: Proceed as above. Use $\epsilon = \epsilon_0 \lambda^-$.

\item Bounding $\norm{\frac{1}{\alpha} \sum_t \lt \vt{}' - \frac{1}{\alpha}\E[\sum_t \lt \vt{}']}_2$: Apply Lemma \ref{lem:cross_subg} with $\bx_t = \lt$, $\by_t = \vt$, $N \equiv \alpha$.  Then $K_x = c\sqrt{\lambda^+}$ and $K_y = c\sqrt{\lambda_v^+}$, thus $K_x K_y = c \sqrt{\lambda^+  \lambda_v^+}$ 
Set $t = \epsilon_{0, v} \lambda^-$ with $\epsilon_{0,v}= c \sqrt{\frac{\lambda_v^+}{\lambda^-} f} \sqrt{\frac{n}{\alpha}}$.
Then,
\begin{align*}
\Pr\left(\norm{\frac{1}{\alpha} \sum_t \lt \vt{}' - \frac{1}{\alpha}\E[\sum_t \lt \vt{}']}_2 \leq \epsilon_{0, v} \lambda^-\right) \geq 1- 2\exp(-c n)
\end{align*}

\item Bounding $\norm{\frac{1}{\alpha} \sum_t \vt \vt{}' - \frac{1}{\alpha}\E[\sum_t \vt \vt{}']}_2$: Apply \cite[Theorem 5.39]{vershynin} and proceed as in the previous part. Use $\epsilon = \epsilon_{1,v} \lambda^-$ with $\epsilon_{1,v} = \frac{\lambda_v^+}{\lambda^-}\sqrt{\frac{n}{\alpha}}$.

\end{enumerate}
Combining the above bounds (except the third one), we conclude that, w.p. $\ge 1- 10 \exp(-cn)$,
\begin{align}
\|\D - \D_0 - \E[\D - \D_0]\|_2 \le 5 \epsbnd_{sG} := c    \max\left(\frac{\lambda_v^+}{\lambda^-},f \right)   \sqrt{\frac{n}{\alpha}}
\label{def_epsbnd}
\end{align}
Everything else remains the same.
\end{proof}

\bibliographystyle{IEEEbib} 
\bibliography{tipnewpfmt_kfcsfullpap}

\begin{thebibliography}{10}

\bibitem{hong_balzano}
D.~Hong, L.~Balzano, and J.~A. Fessler,
\newblock ``Towards a theoretical analysis of pca for heteroscedastic data,''
\newblock in {\em Allerton Conf. Comm. Control Comput.}, 2016.

\bibitem{nadler}
B.~Nadler,
\newblock ``Finite sample approximation results for principal component
  analysis: A matrix perturbation approach,''
\newblock {\em Ann. Statist.}, vol. 36, no. 6, 2008.

\bibitem{spiked_cov}
Iain~M Johnstone,
\newblock ``On the distribution of the largest eigenvalue in principal
  components analysis,''
\newblock {\em Ann. Statist.}, pp. 295--327, 2001.

\bibitem{onlinePCA3}
I.~Mitliagkas, C.~Caramanis, and P.~Jain,
\newblock ``Memory limited, streaming pca,''
\newblock in {\em Adv. Neural Info. Proc. Sys. (NIPS)}, 2013, pp. 2886--2894.

\bibitem{streamingPCA2}
Prateek Jain, Chi Jin, Sham~M Kakade, Praneeth Netrapalli, and Aaron Sidford,
\newblock ``Streaming pca: Matching matrix bernstein and near-optimal finite
  sample guarantees for ojas algorithm,''
\newblock in {\em 29th Annual Conference on Learning Theory}, 2016, pp.
  1147--1164.

\bibitem{corpca_nips}
N.~Vaswani and H.~Guo,
\newblock ``Correlated-pca: Principal components' analysis when data and noise
  are correlated,''
\newblock in {\em Adv. Neural Info. Proc. Sys. (NIPS)}, 2016.

\bibitem{cor_noise_gillberg}
Jussi Gillberg, Pekka Marttinen, Matti Pirinen, Antti~J Kangas, Pasi Soininen,
  Mehreen Ali, Aki~S Havulinna, Marjo-Riitta J{\"a}rvelin, Mika Ala-Korpela,
  and Samuel Kaski,
\newblock ``Multiple output regression with latent noise,''
\newblock {\em J. Mach. Learn. Res.}, 2016.

\bibitem{rrpcp_perf}
C.~Qiu, N.~Vaswani, B.~Lois, and L.~Hogben,
\newblock ``Recursive robust pca or recursive sparse recovery in large but
  structured noise,''
\newblock {\em IEEE Trans. Info. Th.}, pp. 5007--5039, August 2014.

\bibitem{rrpcp_aistats}
J.~Zhan, B.~Lois, H.~Guo, and N.~Vaswani,
\newblock ``{Online (and Offline) Robust PCA: Novel Algorithms and Performance
  Guarantees},''
\newblock in {\em Intnl. Conf. Artif. Intell. Stat. (AISTATS)}, 2016, long
  version: ArXiv: 1601.07985.

\bibitem{beck2009fast}
A.~Beck and M.~Teboulle,
\newblock ``Fast gradient-based algorithms for constrained total variation
  image denoising and deblurring problems,''
\newblock {\em IEEE Trans. Image Process.}, vol. 18, no. 11, pp. 2419--2434,
  2009.

\bibitem{rel_perturb}
Ren-Cang Li,
\newblock ``Relative perturbation theory: Ii. eigenspace and singular subspace
  variations,''
\newblock {\em SIAM J. Matrix Anal. Appl.}, vol. 20, no. 2, pp. 471--492, 1998.

\bibitem{versh_cov_est}
R.~Vershynin,
\newblock ``How close is the sample covariance matrix to the actual covariance
  matrix?,''
\newblock {\em J. Theoret. Probab.}, pp. 1--32, 2010.

\bibitem{vershynin}
R.~Vershynin,
\newblock ``Introduction to the non-asymptotic analysis of random matrices,''
\newblock {\em Compressed sensing}, pp. 210--268, 2012.

\bibitem{davis_kahan}
C.~Davis and W.~M. Kahan,
\newblock ``The rotation of eigenvectors by a perturbation. iii,''
\newblock {\em SIAM J. Numer. Anal.}, vol. 7, pp. 1--46, Mar. 1970.

\bibitem{tail_bound}
J.~A. Tropp,
\newblock ``User-friendly tail bounds for sums of random matrices,''
\newblock {\em Found. Comput. Math.}, vol. 12, no. 4, 2012.

\bibitem{pr_altmin}
P.~Netrapalli, P.~Jain, and S.~Sanghavi,
\newblock ``Phase retrieval using alternating minimization,''
\newblock in {\em Adv. Neural Info. Proc. Sys. (NIPS)}, 2013, pp. 2796--2804.

\bibitem{rpca}
E.~J. Cand{\`e}s, X.~Li, Y.~Ma, and J.~Wright,
\newblock ``Robust principal component analysis?,''
\newblock {\em J. ACM}, vol. 58, no. 3, 2011.

\bibitem{robpca_nonconvex}
P.~Netrapalli, U~N Niranjan, S.~Sanghavi, A.~Anandkumar, and P.~Jain,
\newblock ``Non-convex robust pca,''
\newblock in {\em Neural Info. Proc. Sys. (NIPS)}, 2014.

\bibitem{rpca_gd}
Xinyang Yi, Dohyung Park, Yudong Chen, and Constantine Caramanis,
\newblock ``Fast algorithms for robust pca via gradient descent,''
\newblock in {\em Neural Info. Proc. Sys. (NIPS)}, 2016.

\bibitem{rmc_gd}
Yeshwanth Cherapanamjeri, Kartik Gupta, and Prateek Jain,
\newblock ``Nearly-optimal robust matrix completion,''
\newblock {\em arXiv preprint arXiv:1606.07315}, 2016.

\bibitem{outlier_pursuit}
H.~Xu, C.~Caramanis, and S.~Sanghavi,
\newblock ``Robust pca via outlier pursuit,''
\newblock {\em IEEE Trans. Inform. Theory}, vol. 58, no. 5, May 2012.

\bibitem{lowrank_altmin}
P.~Netrapalli, P.~Jain, and S.~Sanghavi,
\newblock ``Low-rank matrix completion using alternating minimization,''
\newblock in {\em Symposium on Theory of Computing (STOC)}, 2013.

\bibitem{rrpcp_dynrpca}
P.~Narayanamurthy and N.~Vaswani,
\newblock ``{New Results for Provable Dynamic Robust PCA},''
\newblock {\em arXiv:1705.08948}, 2017.

\bibitem{candes_rip}
E.~Candes,
\newblock ``The restricted isometry property and its implications for
  compressed sensing,''
\newblock {\em C. R. Math. Acad. Sci. Paris. Serie I}, pp. 589--592, 2008.

\bibitem{freedman}
J.~A. Tropp,
\newblock ``Freedmans inequality for matrix martingales,''
\newblock {\em Electron. Commun. Probab}, vol. 16, pp. 262--270, 2011.

\end{thebibliography}

\end{document}